\documentclass{article}

\usepackage[final]{acl}

\title{FaithLM: Towards Faithful Explanations for
Large Language Models}

\usepackage{microtype}
\usepackage{graphicx}
\usepackage{subfigure}
\usepackage{booktabs} % for professional tables
\usepackage{hyperref}
\usepackage{times}

\usepackage[utf8]{inputenc} % allow utf-8 input
\usepackage[T1]{fontenc}    % use 8-bit T1 fonts
\usepackage{hyperref}       % hyperlinks
\usepackage{url}            % simple URL typesetting
\usepackage{booktabs}       % professional-quality tables
\usepackage{amsfonts}       % blackboard math symbols
\usepackage{nicefrac}       % compact symbols for 1/2, etc.
\usepackage{microtype}      % microtypography
\usepackage{xcolor}         % colors

\usepackage{makecell} 
\usepackage{amsmath}
\usepackage{amssymb}
\usepackage{mathtools}
\usepackage{amsthm}

\usepackage{url}
\usepackage{bbm}
\usepackage{amssymb}
\usepackage{enumitem}
\usepackage{algorithm, algorithmic}
\usepackage{wrapfig}
\usepackage{mathrsfs}
\usepackage{multirow}
\usepackage{amsthm}
\usepackage{caption}

\usepackage[textsize=tiny]{todonotes}

\newtheorem{theorem}{Theorem}
\newtheorem{corollary}{Corollary}

\def\Algnameline{\textit{\textcolor{darkred}{Faith}ful L\textcolor{darkred}{LM} Explainers}}

\def\Algnameabbr{\texttt{FaithLM}}

\definecolor{darkred}{RGB}{255,0,0}
\definecolor{darkblue}{RGB}{0,0,180}

\author{%
Yu-Neng Chuang$^{1}$\thanks{Correspond to Yu-Neng Chuang  <ynchuang@rice.edu>, Vladimir Braverman <vova@cs.jhu.edu> and Xia Hu <xia.hu@rice.edu>} \quad Guanchu Wang$^{1}$ \quad \textbf{Chia-Yuan Chang}$^2$ \quad \textbf{Ruixiang Tang}$^1$ \\ \textbf{Shaochen Zhong}$^{1}$ \quad \textbf{Fan Yang}$^{3}$ \quad \textbf{Mengnan Du}$^{4}$ \quad \textbf{Xuanting Cai}$^{5}$ \\ \quad \textbf{Vladimir Braverman}$^{56}$ \quad \textbf{Xia Hu}$^1$\\
Rice University$^1$ \quad Texas A\&M University$^2$ \quad Wake Forest University$^3$ \\ New Jersey Institute of Technology$^4$ \quad John Hopkins University$^5$ \\ \quad Google Research$^6$ \quad Meta Platforms, Inc.$^7$\\
}

\begin{document}
\maketitle

\begin{abstract}
% Large Language Models (LLMs) have become proficient in addressing complex tasks by leveraging their extensive internal knowledge and reasoning capabilities.
% However, the black-box nature of these models complicates the task of explaining their decision-making processes.
% While recent advancements demonstrate the potential of leveraging LLMs to self-explain their predictions through natural language (NL) explanations, their explanations may not accurately reflect the LLMs' decision-making process due to a lack of fidelity optimization on the derived explanations.
% Measuring the fidelity of NL explanations is a challenging issue, as it is difficult to manipulate the input context to mask the semantics of these explanations.
% To this end, we introduce \Algnameabbr{} for explaining the decision of LLMs with NL explanations. Specifically, \Algnameabbr{} designs a method for evaluating the fidelity of NL explanations by incorporating the contrary hints to the query process. Moreover, \Algnameabbr{} conducts an iterative process to improve the fidelity of derived explanations.
% Experiment results on three datasets from multiple domains demonstrate that \Algnameabbr{} can significantly improve the fidelity of derived explanations, which also provides a better alignment with the ground-truth explanations.

Large language models (LLMs) increasingly produce natural language explanations, yet these explanations often lack faithfulness, and they do not reliably reflect the evidence the model uses to decide. We introduce \Algnameabbr{}, a model-agnostic framework that evaluates and improves the faithfulness of LLM explanations without token masking or task-specific heuristics. \Algnameabbr{} formalizes explanation faithfulness as an intervention property: a faithful explanation should yield a prediction shift when its content is contradicted. Theoretical analysis shows that the resulting contrary-hint score is a sound and discriminative estimator of faithfulness. Building on this principle, \Algnameabbr{} iteratively refines both the elicitation prompt and the explanation to maximize the measured score. Experiments on three multi-domain datasets and multiple LLM backbones demonstrate that \Algnameabbr{} consistently increases faithfulness and produces explanations more aligned with human rationales than strong self-explanation baselines. These findings highlight that intervention-based evaluation, coupled with iterative optimization, provides a principled route toward faithful and reliable LLM explanations. 
Our source code is available at \url{https://github.com/ynchuang/xLLM}.
\end{abstract}
% \vspace{-0.5cm}

% more general framework, existing work is a special case of our work
% highlight the contribution of zero-shot
% During deployment, \Algnameabbr{} simply generates the expectation for the Shapley value; and the standard derivation for the estimation uncertainty.

\section{Introduction}
\label{sec:intro}

Large language models (LLMs) exhibit remarkable performance in various natural language processing tasks, such as the GPT4~\cite{achiam2023gpt}, LLaMA~\cite{touvron2023llama}, and Claude~\cite{claude2023}.
However, these language models are commonly regarded as intricate black-box systems. 
The opacity of their internal mechanisms poses a significant challenge when trying to explain their decision-making process.
The lack of transparency in LLMs, especially in API-accessed LLM services, inferences contradict the practical requirements of stakeholders and are in opposition to regulatory standards in various domains, such as GDPR~\cite{goodman2017european, floridi2019establishing}. The imperative arises to develop explainability mechanisms for LLMs, particularly for their use in high-stakes applications such as healthcare. In this work, we focus on ``LLM explanation", rather than ``LLM reasoning" or ``LLM self-refinement", to interpret the model prediction behaviors after providing the final responses. (More illustrations of their discrepancy are in Section~\ref{sec:exp_vs_reason}).

Numerous studies have attempted to enhance the transparency of decision-making processes in LLMs by providing natural language (NL) explanations. However, this complexity poses challenges to faithfully explain the underlying explanation behind their decisions with natural language sentences. Recent advancements are struggling to generate reliable NL explanations for interpreting LLMs~\cite{ye2022unreliability}. Some work attempt to leverage powerful LLMs~\cite{majumder2021knowledge, chen2023lmexplainer, chen2023models} with auxiliary information to generate NL sentences or heatmap of input tokens as model explanations. Although existing work emerged that LLMs may possess the ability to self-explain~\cite{madsen2024can}, their explanation-generating process usually overlooks the fidelity, a fundamental metric for evaluating the quality of explanations~\cite{chuang2023efficient, wang2023leta}. The derived NL explanations may not faithfully reflect ``why model generate this answers."~\cite{zhao2023explainability, turpin2023language} Some work attempts to leverage chain-of-thought (CoT) reasoning steps as the post-hoc model explanation~\cite{lyu2023faithful, radhakrishnan2023question}. However, these reasoning steps are not considered model explanations in the context of post-hoc explanation~\cite{tanneru2024hardness}, where post-hoc ones particularly focus on providing faithful explanations for a given generated answer. These CoT steps are produced without a thorough fidelity check (i.e., one that involves masking out the key factors) to ensure they genuinely influence the final answer. These steps are only the intermediate results during the LLM prediction and their fidelity remains unknown. A proper fidelity measurement requires masking the critical features or key messages in the explanation and observing the model’s performance afterward~\cite{du2019techniques}, but neither of them are monitored or adopted before claiming CoT reasoning as model explanations.
Measuring the fidelity of NL explanations now become a important but challenging issue, as we can monitor and optimization the explanation generation process based on fidelity improvement. The ones may provide crucial information beyond the input context, but the faithful information may appear in semantic levels, making it hard to measure by manipulating the tokens for fidelity measurement.

\begin{table*}[t!]
\centering
\resizebox{\textwidth}{!}{
\begin{tabular}{p{4.cm} p{4.2cm} p{8.5cm}}
    \toprule
    {\footnotesize Question and LLM Answer} & {\footnotesize \text{Faithful NL Explanation}} & {\footnotesize Question conditioned with Contrary NL Explanation}\\
    \midrule
    \makecell[l]{\footnotesize \textbf{Question:} Can the positive \\ \footnotesize pole from two magnets pull \\ \footnotesize each other closer? \\ \footnotesize \textbf{Original Answer:} \textbf{\color{darkblue}{No}}} 
    & \makecell[l]{\footnotesize Each magnet has a positive \\ \footnotesize pole and a negative pole, \\ \footnotesize and \color{darkblue}{\textbf{similar}} \color{darkblue}{\textbf{poles push}} \\ \footnotesize \color{darkblue}{\textbf{each other away.}} } 
    & \makecell[l]{\footnotesize \textbf{Question:} Each magnet has a positive pole and a negative \\ \footnotesize pole, and {\color{red}{\textbf{similar poles}}} {\color{red}{\textbf{pull each other closer.}}} Can the \\ \footnotesize positive pole from two magnets pull each other closer? \\ \footnotesize\textbf{New Answer:} \textbf{\color{red}{Yes}} } \\
    \bottomrule
\end{tabular}
}
% \vspace{-0.1cm}
\caption{An example of measuring fidelity of NL explanations. The LLM first answers the question in \textbf{\textcolor{darkblue}{No}}. Given a faithful explanation \textit{``similar poles \textit{pull} each other away,"}  with its contrary NL explanation, the LLM changes the answer from \textbf{\textcolor{darkblue}{No}} to \textbf{\textcolor{darkred}{Yes}} when introduced contrary NL explanation as an extra condition to LLM. This example indicates that the contrary hint interrupts the LLM's original prediction process with faithful information, demonstrating that the information in an original explanation is faithful and aligns with LLM's initial answer.}
\vspace{-0.2cm}
\label{fig:prelim_example}
\end{table*}

To overcome this challenge, we assess the fidelity of natural language explanations by treating faithfulness as a causal property: if an explanation captures information or hints the model actually uses, then intervening on that information should predictably change the model’s output. Concretely, given an explanation, we construct a contrary hint that expresses the opposite semantics and append it to the input. We then measure fidelity as the resulting prediction shift relative to the original output. A large, directionally consistent change (for example, from \textbf{\textcolor{darkblue}{No}} to \textbf{\textcolor{darkred}{Yes}}) indicates that the explanation contained decision-relevant content that the contrary hint displaced. This evaluation mirrors recent work~\cite{chen2025reasoning} that tests whether models use and verbalize externally provided hints, and that reads faithfulness from sensitivity under controlled prompt interventions. Our procedure applies the same principle to explanation content, enabling a simple, model-agnostic fidelity measure for free-text rationales

% we propose a method to measure the fidelity of NL explanations.
% We give an example to convey the motivation in Table~\ref{fig:prelim_example}. Specifically, the fidelity of an explanation can be measured by leveraging its contrary hint as extra conditions of the input context, and observing the LLM's output difference compared with its initial output. Here, a contrary hint refers to a statement with opposite semantics to the original explanation.
% By incorporating the contrary hint to the input context, we can identify an explanation as high fidelity if there is a significant change in the LLM's output, such as from \textbf{\textcolor{darkblue}{No}} to \textbf{\textcolor{darkred}{Yes}}.
% This change indicates that the crucial information present in the original explanation is substituted with the opposite meaning context in the contrary hint, where the crucial information is essential to the LLM's decision-making process. Based on this observation, we propose to extend the applicability of fidelity to the evaluation on NL explanations. This extension follows the integration of contrary hints to represent the concepts of masking important features in traditional fidelity measurement.

Building upon this new fidelity measurement, we introduce \Algnameline{}~(\Algnameabbr{}) to generate faithful NL explanations for LLMs.
Specifically, \Algnameabbr{} adopts LLMs as explainer to generate the NL explanations and explanation trigger prompts, and iteratively optimizes the derived NL explanations and trigger prompts with the goal of fidelity enhancement.
During the iterative process, \Algnameabbr{} computes the fidelity of each derived explanation and optimized prompt based on our proposed fidelity measurement method, and progressively improves their fidelity through in-context learning. We conducted the experiments on four different LLMs under three datasets. \Algnameabbr{} achieves significantly higher fidelity in generating NL explanations and more closely matched the golden explanations compared with state-of-the-art baseline methods. Our contributions can be summarized as follows:

\begin{itemize}[leftmargin=*]
    \itemsep=-1.8pt
    % \item \textbf{Fidelity of NL Explanations:} We propose the way to measure the fidelity of NL explanations by introducing a contradictory explanation and observing how the LLM’s output changes.
    % \item \textbf{Faithful LLM Explainers:} \Algnameabbr{} improves the fidelity of NL explanations, aiming at faithfully explaining the decision-making process of LLMs.
    % \item \textbf{Fidelity and Truthfulness:} Experimental results show that \Algnameabbr{} can improve the fidelity of NL explanations, revealing a better alignment with the ground-truth explanations than SoTA baselines. 
    \item \textbf{Intervention-based fidelity.} We define fidelity as prediction sensitivity to a \emph{contrary hint} that semantically opposes the explanation.
    \item \textbf{Faithful LLM Explainers.} Our method generates contrary hints, computes fidelity from outputs, and iteratively refines the prompt and explanation to increase faithfulness.
    \item \textbf{Empirical gains.} Experimental results show that \Algnameabbr{} raises measured fidelity and improves alignment with human explanation over baselines.

\end{itemize}

\section{Preliminaries}
\subsection{Notations and Objectives}
We aim to explain the decisions of arbitrary targeted LLMs $f(\cdot)$ with NL explanations in a post-hoc manner. Given an input $\boldsymbol{X}$, the targeted LLMs generate an output $\boldsymbol{Y} = f(\boldsymbol{X})$. Our objective is to produce an NL explanation $\mathcal{E}_{\text{NL}}$ that faithfully explains the reasons behind the prediction of $\boldsymbol{Y} = f(\boldsymbol{X})$.
In this work, we employ an LLM as the explainer $\textsl{g}(\cdot)$ to generate the NL explanation $\mathcal{E}_{\text{NL}} = \textsl{g}(\cdot ~|~ \boldsymbol{X}, \boldsymbol{Y})$.
However, the directly generated $\mathcal{E}_{\text{NL}}$ under single-forward passing may not be faithful and accurate, degrading the user's trust in the prediction made by the targeted LLM. The consistency between $f(\cdot)$ and $\textsl{g}(\cdot)$ is ensured through an iterative optimization process monitored by fidelity scores.
To this end, the explainer $\textsl{g}(\cdot)$ to generate \textbf{more faithful NL explanations regarding the decision of $f(\cdot)$ in a post-hoc manner}, where $f(\cdot)$ can be either closed-source or open-source LLMs. FaithLM targets realistic post-hoc interpretability for API-only or black-box LLMs, so we separate the \emph{target} model $f(\cdot)$ that answers from the \emph{explainer} $g(\cdot)$ that proposes $E_{NL}$ (and $\lnot E_{NL}$).

\subsection{Difference between LLM Explanation and Chain-of-thoughts}
\label{sec:exp_vs_reason}
Due to the limited accessibility of LLM APIs, recent research on LLM explanations has largely relied on post-hoc explanation approaches~\citep{chen2023lmexplainer}. However, some studies conflate 'LLM reasoning' and ``LLM self-refinement' with 'LLM explanations' when discussing these post hoc LLM explanations, even though these three terms are not identical and with different goals. We illustrate the difference as follows.

\paragraph{LLM reasoning and Chain-of-thoughts} refers to the internal process the model undergoes when it encounters a query or instruction, such as weighting probabilities and generating words step by step plus verification, with the goal of improving performance on reasoning tasks.
Some advantages rely on providing chain-of-thought (CoT) reasoning~\citep{lanham2023measuring, radhakrishnan2023question, chen2023models, wang2022self} to present the hidden inference steps that the model goes through. These studies show that CoT~\cite{manuvinakurike2025thoughts} can improve reasoning performance, but does not necessarily provide an explanation or even count as explanations of how or why an LLM arrives at its answers~\citep{tanneru2024hardness}, where “good fidelity” in the series of work is typically defined by the alignment between the content of CoT and the final answer~\citep{lyu2023faithful, radhakrishnan2023question}.
Another line of work leverages self-refinement techniques~\citep{lightman2023let, madaan2024self, tian2024toward}, which employ self-reasoning or knowledge supervision as feedback, to iteratively enhance reasoning performance. Although these advancements introduce robust self-feedback loops that effectively boost reasoning accuracy, the “feedback” during optimization is neither necessarily faithful nor equivalent to LLM explanations~\citep{tanneru2024hardness}. Notably, this feedback may be wrong yet still guide LLMs toward a correct reasoning direction. Due to its non-stationary nature, it yields non-faithful outputs when treated as an LLM explanation, which is also very distinct from the goal of ``LLM explanation" tasks.

\paragraph{LLM Explanations.} Unlike LLM reasoning and self-refinement, \textit{LLM explanation} focuses on clarifying why the model provides a particular answer after generating its final decision~\citep{siegel2024probabilities}. The concept of fidelity in an LLM explanation~\citep{du2019techniques, zhao2023explainability}, which differs from LLM reasoning and self-refinement, refers to whether the model’s prediction would change if the key knowledge provided explanation were removed. If removing the knowledge causes a drastic change in the model’s prediction, we can conclude that the derived explanation is faithful to the LLMs prediction (i.e., the actual reasons that results the predictions). In this work, we focus on LLM explanation, rather than LLM reasoning or self-refinement.

% Recent research on LLM explanations has primarily focused on post-hoc approaches due to the inaccessibility of LLM weights and architectures.
% One group of studies calculates importance scores for specific tokens~\cite{lopardo2023faithful, huang2023can}, and the other line of work generates natural language sentences~\cite{kumar2020nile, chen2021kace, chen2023lmexplainer, menon2023mantle} and CoT reasoning~\cite{lyu2023faithful, chen2023models} as explanations by leveraging LLMs. Despite the advanced capability of powerful LLMs to generate NL explanations, issues with unreliability and non-fidelity persist. While studies like Chen et al.~\cite{chen2024selfie} utilize internal model weights aiming for reliable explanations generation, accessing these weights is challenging with closed-source LLMs.
% Considering the non-accessibility of LLMs, such as closed-source LLMs, our goal is to faithfully generate NL explanations instead of exploring the internal mechanisms or neurons of LLMs to generate explanations.

\subsection{Limitations of Traditional Fidelity Measurement on NL Explanations}
\label{sec:prelim_fidelity}
The fidelity metric measures the fidelity of the given explanation, which is broadly applicable when ground-truth explanations are unavailable.
In the NLP scenario, fidelity has been used to evaluate the heatmap-formatted explanations~\cite{lopardo2023faithful, huang2023can}, where the heatmap one highlights the important tokens of the input. Specifically, fidelity evaluates the explanation by removing the important tokens from the input $\boldsymbol{X}$ and checking the prediction difference of the targeted LLM. Following the definition of fidelity~\cite{miro2024comprehensive}. Given a sequence of tokens $\boldsymbol{I} = \{t_1, \cdots, t_M \} \subseteq \mathcal{E}_{\text{NL}}$, which is identified as an important component of explanation to the prediction of a targeted LLM $\boldsymbol{Y} = f(\boldsymbol{X})$. Following the traditional fidelity definition, the fidelity can be estimated as:
\begin{equation}
    \notag \mathrm{Fidelity} = f(\boldsymbol{X}) - f(\boldsymbol{X} \setminus \boldsymbol{I}),
\end{equation}
where "$\boldsymbol{X} \setminus \boldsymbol{I}$" denotes token removal from $\boldsymbol{X}$ in $\boldsymbol{I}$. 

If important component $\boldsymbol{T}$ achieves higher fidelity, this demonstrates that $\mathcal{E}_{\text{NL}}$ comprises the crucial tokens that significantly influence the predictions of the targeted LLMs. 
However, it is challenging to evaluate the fidelity of NL explanations throughout the fidelity defined above, as the critical components in NL explanations may not contain in the input context $\boldsymbol{X}$.Some work~\cite{lanham2023measuring} attempts to measure fidelity by modifying the output chain-of-thought (CoT) reasoning to overcome this challenge. However, altering only the output does not guarantee changes in the model’s pre-filling probability and may therefore meet self-consistency, but rather than the definition of fidelity. Thus, we cannot simply remove or modify critical components from the question following the traditional definition. Unlike the previous approaches, we propose a solution to systematically address this obstacle by removing the critical components from the semantic level instead of the token level.

\section{\Algnameabbr{}: The Explainer LLM Framework}
In this section, we systematically introduce the generative explanation framework, \Algnameabbr{}, which derives faithful explanations in natural language format. The derived explanations are expected to accurately reflect the predictive decision-making process of targeted LLMs with high fidelity after optimizing under \Algnameabbr{}. 
% Specifically, we first introduce the Fidelity Evaluator in Section~\ref{sec:fid} to quantitatively assess the fidelity of the NL explanations. Subsequently, \Algnameabbr{} iteratively optimizes the fidelity of the derived explanations and searches for improved explanation trigger prompts using the proposed Fidelity Evaluator. The optimization frameworks for these processes are detailed in Sections~\ref{sec:exp_opt} and~\ref{sec:tri_opt}. 

% \subsection{Fidelity Evaluator for Natural Language Explanations} 
% \label{sec:fid}

% As mentioned in Section~\ref{sec:prelim_fidelity}, we introduce the \emph{Fidelity Evaluator} to directly assess the fidelity of NL explanations shown in Figure~\ref{fig:fed}. 

% \paragraph{Fidelity of NL Explanations.}
% \label{sec:nle_fidelity}

\begin{figure}
  \begin{center}
    \includegraphics[width=0.48\textwidth]{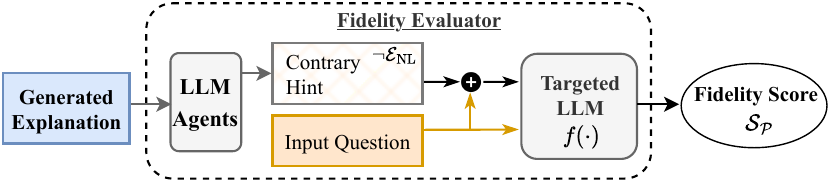}
  \end{center}
  \vspace{-2pt}
  \caption{The fidelity evaluation with hint. The evaluator calculates the fidelity scores of the derived explanations based on its contrary hints.}
  \vspace{-3pt}
  \label{fig:fed}
\end{figure}

\begin{figure*}[t!]
\centering
% \subfigcapskip=-1mm
\subfigure[Fidelity-enhanced Explanation]{
\centering
	\begin{minipage}[t]{0.48\linewidth}
		\includegraphics[width=0.99\linewidth]{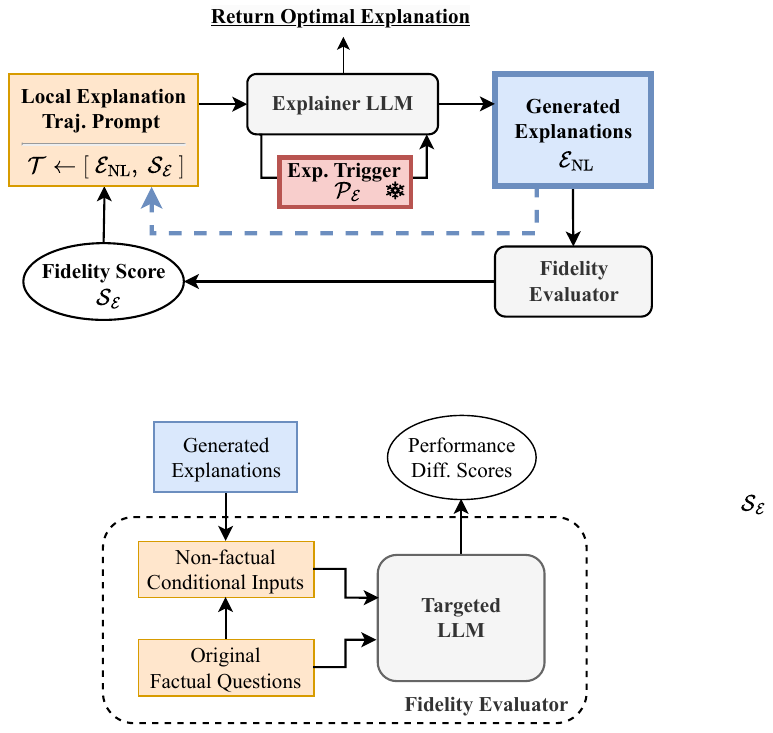}
	\end{minipage}%
        \label{fig:over-a}
}
\!\!\!\!
\subfigure[Trigger Prompt Optimization]{
% \centering
	\begin{minipage}[t]{0.48\linewidth}
		\includegraphics[width=0.99\linewidth]{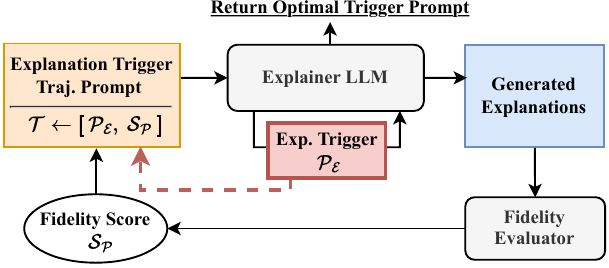}
	\end{minipage}%
        \label{fig:over-b}
}
% \vspace{-0.3cm}
\caption{An overview of \Algnameabbr{} framework for two differenent optimization objectives.
The {\color{darkblue}{blue dotted line}} reveals the trajectory to optimize the NL explanation (Section~\ref{sec:exp_opt}), and the {\color{darkred}{red dotted line}} indicates the trajectory of the explanation trigger prompt optimization  (Section~\ref{sec:tri_opt}). ``Traj. Prompt" denotes the trajectory system prompt shown in Appendix~\ref{appendix:prompt_xllm}.}
% \vspace{-0.3cm}
\label{fig:over}
\end{figure*}

\subsection{Fidelity via Contrary Hint Interventions}

To reduce penalization from irrelevant tokens, FaithLM computes the probability gap $\Delta p(Y|X)$ only over tokens with non-zero attribution in $f(\cdot)$’s saliency map, thereby isolating fidelity shifts caused by semantically relevant regions.

\paragraph{Contrary Hint Interventions} We evaluate the fidelity of a natural language explanation \(\mathcal{E}_{\text{NL}}\) by treating faithfulness as an intervention property on model inference, in line with recent work that tests faithfulness via controlled changes to the information a model conditions on \citep{chen2025reasoning, lanham2023measuring}. Let \(f(\mathbf{X})\) denote the model prediction on input \(\mathbf{X}\). We construct a \emph{contrary hint} \(\neg \mathcal{E}_{\text{NL}}\), defined as a statement whose semantics are opposed to \(\mathcal{E}_{\text{NL}}\) (for example, “similar poles pull each other closer’’ as the contrary of “similar poles push each other away’’). We then condition the model on this contrary hint and read out a fidelity score from the intervention-induced prediction shift:
\[
\mathcal{S}_{\mathcal{E}} := f(\mathbf{X}) - f(\mathbf{X}\mid \neg \mathcal{E}_{\text{NL}}).
\]
If \(\mathcal{E}_{\text{NL}}\) captures information the model actually uses, inserting \(\neg \mathcal{E}_{\text{NL}}\) produces a measurable and directionally consistent change in the output distribution, often sufficient to alter the discrete decision. This operationalization supports free-text explanations without token masking or task-specific heuristics. In practice, our \emph{Fidelity Evaluator} (i) prompts a strong LLM to produce \(\neg \mathcal{E}_{\text{NL}}\) from \(\mathcal{E}_{\text{NL}}\) under a fixed instruction template and (ii) estimates \(\mathcal{S}_{\mathcal{E}}\) from output logits or calibrated probabilities, using the sensitivity to the contrary-hint intervention as the fidelity signal.

\vspace{0.1cm}
\noindent\textbf{Relation to prior faithfulness evaluations.}
Our evaluation protocol follows the intervention-based paradigm introduced by prior work~\citep{chen2025reasoning} to test whether verbalized rationales genuinely reflect a model’s internal decision process~\citep{chen2025reasoning}. That work manipulates the presence and semantics of hints in prompts and measures whether models both \emph{use} and \emph{verbalize} those hints, revealing that stated chains of thought are often unfaithful to underlying reasoning. We adopt the same causal-testing philosophy but apply it to explanation content rather than reasoning text. Specifically, the contrary hint serves as a controlled semantic intervention on the information encoded in the explanation, and the resulting change in model prediction quantifies its \emph{fidelity}. In this formulation, faithfulness is defined as prediction sensitivity: an explanation is faithful if altering its truth value changes the model’s decision. Consequently, our contrary-hint score provides a direct, model-agnostic instantiation of the evaluation approach advocated by~\citet{chen2025reasoning}, complementing prior intervention tests on CoT text~\citep{lanham2023measuring} by operating at the level of meaning rather than surface form.

\vspace{0.1cm}
\begin{theorem}[\textbf{Latent-Context Intervention Validity for Faithfulness}]
\label{thm:intervention}
Let $f:\mathcal{X}\!\times\!\mathcal{C}\!\to\!\Delta(\mathcal{Y})$ be a language model mapping an input $X$ and latent context $C$ to a predictive distribution over an output space $\mathcal{Y}$. 
Let $E_{NL}$ denote a natural-language explanation of $f(X;C)$, and let $\lnot E_{NL}$ denote its contrary hint. 
Assume that $E_{NL}$ asserts a proposition about a semantic factor $S_E=s(X,C)$, where $s(\cdot)$ extracts the decision-relevant concept, which is latent or retrieved, that the explanation verbalizes. 
Conditioning on $\lnot E_{NL}$ is equivalent to intervening on this factor while holding $(X,C)$ fixed, i.e.,
$f(X;C\mid\lnot E_{NL}) = f(X;C\mid do(S_E\!\leftarrow\!\bar s))$ 
for some contradictory value $\bar s$, and predictions are invariant to any irrelevant text $R$, so $f(X;C)=f(X\cup R;C)$. 
Defining $S_E(X;C)=D(f(X;C),f(X;C\mid\lnot E_{NL}))$, where $D$ is any strictly proper divergence, we have
\vspace{0.1cm}
\begin{align}
    \notag S_E(X;C)=0 &\iff E_{NL}\text{ is non-faithful for }f(X;C),\\
    \notag S_E(X;C)>0 &\iff E_{NL}\text{ is faithful for }f(X;C).
\end{align}
\vspace{0.1cm}
Hence, the contrary-hint score $S_E$ constitutes a valid empirical estimator of faithfulness when the decision-relevant content is not contained in the observed input $X$ but arises from latent or retrieved context.
\end{theorem}

\vspace{0.1cm}
\noindent\textbf{Intuition of the theoretical foundation.}
Theorem~\ref{thm:intervention} formalizes this causal perspective: if the explanation encodes information that lies on a causal path to the model’s output, then intervening with a contradictory hint necessarily shifts the predictive distribution, yielding a positive contrary-hint score. Conversely, if the explanation is merely correlational or decorative, the intervention leaves the model unchanged and the score remains zero. Thus, the contrary-hint fidelity measure $S_E$ causal faithfulness by directly testing whether the model’s output is \emph{functionally dependent} on the semantics expressed in its own explanation. More discussions and proof are in Appendix~\ref{apdx:proof} and Corollary~\ref{cor:robustness}.

\subsection{\Algnameabbr{} on Fidelity-enhanced Explanation}
\label{sec:exp_opt}
In this section, we introduce an iterative framework designed to progressively enhance the fidelity of NL explanations. The primary goal of \Algnameabbr{} here is to generate faithful NL explanations with iterative fidelity-enhanced optimization.

\paragraph{Fidelity-enhanced Explanation.} 
The framework of fidelity-enhanced explanation is illustrated in Figure~\ref{fig:over-a}. Since the initial explanation may be unreliable and unfaithful, we propose a fidelity-enhanced optimization approach designed to progressively generate explanations with higher fidelity.
We aim to explain the response $\boldsymbol{Y}$ produced by the targeted LLM $f(\cdot)$ in response to the given input queries $\boldsymbol{X}$ with NL explanations $\mathcal{E}_{\text{NL}}$ following the goal of fidelity enhancement.
In the first round of enhancement, the LLM explainer generates NL explanations $\mathcal{E}_{\text{NL}}$ following a given human-crafted explanation trigger prompt $\mathcal{P}_\mathcal{E}$ provided in Appendix~\ref{appendix:prompt_meta}. The explanations are then generated by the explainer $\textsl{g}( \mathcal{P}_\mathcal{E} | ~\boldsymbol{X}, \boldsymbol{Y})$.
Starting from the second round till converge, \Algnameabbr{} collects a trajectory $\mathcal{T}$ with the NL explanations $\mathcal{E}_{\text{NL}}$ and their corresponding fidelity scores $\mathcal{S}_\mathcal{E}$ generated by Fidelity Evaluator. The collection process can be represented as $\mathcal{T} \leftarrow \{ \mathcal{T}, ~[ \mathcal{E}_{\text{NL}}, ~\mathcal{S}_\mathcal{E}] \}$, where $\mathcal{T}$ initially starts as an empty trajectory. Following this trajectory, the LLM explainer generates new explanations with the goal of achieving higher fidelity scores in subsequent iterations. This process is guided by the system prompts detailed in Figure~\ref{appendix:fig:exp}. 

The trajectory $\mathcal{T}$ is continuously updated by incorporating each newly derived explanation with its assessed fidelity score until the convergence.
Regardless of any given explanation trigger prompts $\mathcal{P}_\mathcal{E}$, \Algnameabbr{} can all systematically guide the generation of NL explanations, progressively improving fidelity scores by following the reference path established in the trajectory.

\vspace{0.2cm}
\noindent\textbf{Algorithm of Fidelity-enhanced Explanation.}
\label{sec:opt_algo}
The outline of \Algnameabbr{} for Fidelity-enhanced Explanation is detailed in Algorithm~\ref{alg:xllm-exp}.
Specifically, in the first iteration, \Algnameabbr{} generates the NL explanations using the human-craft prompts~(line 1). starting from the second iteration till the convergence or optimization ends, \Algnameabbr{} estimates the fidelity of the derived NL explanations~(line 4). Then, we incorporate the explanation and its corresponding fidelity score to the trajectory~(line 5), and update the explanations with the goal of achieving higher fidelity scores in subsequent iterations~(line 6). The iteration terminates at a predetermined step or ceases earlier as soon as \Algnameabbr{} observes a flipping performance from the targeted LLM $f(\cdot)$.

\begin{algorithm}[t]
\small
\caption{\footnotesize Fidelity-enhanced Explanation}
\label{alg:xllm-exp}
\textbf{Input:} Input $\boldsymbol{X}$, output $\boldsymbol{Y}$, targeted LLMs $f(\cdot)$, human-crafted prompt $\mathcal{P}_\mathcal{E}$, and LLM explainer $\textsl{g}(\cdot)$.\\
\textbf{Output:} NL explanation $\mathcal{E}_{\text{NL}}$.\\
\vspace{-4mm}
\begin{algorithmic}[1]
\STATE $\mathcal{E}_{\text{NL}} \sim\textsl{g}(\mathcal{P}_\mathcal{E} ~|~ \boldsymbol{X}, \boldsymbol{Y})$
% \STATE Given explanation-oriented trajectory prompt $\mathcal{T}$
\STATE $\mathcal{T} = \varnothing$
\WHILE{ \textit{steps not end} \textbf{and} \textit{decision not flips} }
    \STATE Estimate the fidelity score $\mathcal{S}_\mathcal{E}$ of $\mathcal{E}_{\text{NL}}$ %from the Fidelity Evaluator
    \STATE Append $\mathcal{T} \leftarrow ~\mathcal{T} \cap [\mathcal{E}_{\text{NL}},~ \mathcal{S}_\mathcal{E} ]$
    \STATE Update $\mathcal{E}_{\text{NL}} \sim \textsl{g}( \mathcal{T} ~|~ \boldsymbol{X}, \boldsymbol{Y})$
\ENDWHILE
\end{algorithmic}
\end{algorithm}
\raggedbottom

% \vspace{-0.2cm}
\subsection{\Algnameabbr{} on Trigger Prompt Optimization}
\label{sec:tri_opt}

Despite the success of enhancing fidelity in Section~\ref{sec:exp_opt}, the low quality of the explanation trigger prompts $\mathcal{P}_\mathcal{E}$ may still hinder the optimization process of receiving a high-fidelity explanation.
Given that the unknown preference for prompts from LLMs, human-crafted trigger prompts used in Fidelity-enhanced Explanation Optimization might lead to sub-optimal fidelity enhancement in the derived explanations. In this section, we hereby propose a new optimization pipeline under \Algnameabbr{}, aiming to optimize the trigger prompt $\mathcal{P}_\mathcal{E}$ for generating NL explanations with higher fidelity scores as the LLM explanations of input each input query.

\paragraph{Trigger Prompt Optimization.}
The framework of Trigger Prompt Optimization is shown in Figure~\ref{fig:over-b}. The framework aims to optimize the trigger prompt to generate NL explanations with higher fidelity. 
Different from the optimization goal in Section~\ref{sec:exp_opt}, the trajectory in this task collects the trigger prompts $\mathcal{P}_\mathcal{E}$ and their fidelity scores $\mathcal{S}_\mathcal{P}$. The trajectory is constructed by the system optimization prompts detailed in Figure~\ref{appendix:fig:trg}. 

To estimate the fidelity score for a trigger prompt, \Algnameabbr{} first adopts the randomly human-crafted trigger prompt to guide the LLM explainers to generate NL explanations, and then utilize the Fidelity Evaluator to assess the fidelity of the derived explanation. The final estimated score is averaged by the fidelity score $\mathcal{S}_{\mathcal{E}_{i}}$ of the hold-out dataset $(\boldsymbol{X}_i, \boldsymbol{Y}_i) \in \mathcal{D}$.
Formally, the fidelity score for a trigger prompt $\mathcal{P}_\mathcal{E}$ is as follows:
\begin{equation}
    \mathcal{S}_{\mathcal{P}} = \mathbb{E}_{\mathcal{E}_{i} \sim \textsl{g}(\mathcal{P}_{\mathcal{E}} | \boldsymbol{X}_i, \boldsymbol{Y}_i)} \big[\mathcal{S}_{\mathcal{E}_{i}}\big],
\end{equation}
where $\mathcal{S}_{\mathcal{E}{i}}$ represents the fidelity score of the explanation $\mathcal{E}_{i}$, which is generated by $\textsl{g}(\mathcal{P}_{\mathcal{E}} | \boldsymbol{X}_i, \boldsymbol{Y}_i)$, as assessed by the Fidelity Evaluator. 

\begin{algorithm}[t]
\small
\caption{{\footnotesize Trigger Prompt Optimization.}}
\label{alg:xllm-trg}
\textbf{Input:} Hold-out dataset $\mathcal{D}$, Targeted LLMs $f(\cdot)$, and LLM explainers $\textsl{g}(\cdot)$.\\
\textbf{Output:} Optimal explanation trigger prompt $\mathcal{P}_\mathcal{E}$.\\
\vspace{-4mm}
\begin{algorithmic}[1]
% \STATE Given an initial explanation trigger prompt $\mathcal{P}_\mathcal{E}$
% \STATE Build up trigger-oriented trajectory prompt $\mathcal{T}$
\STATE Initialize human-crafted $\mathcal{P}_\mathcal{E}$
\STATE Initialize $\mathcal{T} = \{\varnothing$\}
\WHILE{ (Steps Not End) }
    \FOR{$(\boldsymbol{X}_i, \boldsymbol{Y}_i) \sim \mathcal{D}$}
        \STATE $\mathcal{E}_{i} \leftarrow \textsl{g}(\mathcal{P}_\mathcal{E} ~|~ \boldsymbol{X}_i, \boldsymbol{Y}_i)$ 
        \STATE Estimate the fidelity score $\mathcal{S}_i$ of $\mathcal{E}_{i}$
    \ENDFOR
    \STATE $\mathcal{S}_{\mathcal{P}} = \mathbb{E}_{\mathcal{E}_{i} \sim \textsl{g}(\mathcal{P}_{\mathcal{E}} | \boldsymbol{X}_i, \boldsymbol{Y}_i)} \big[\mathcal{S}_{\mathcal{E}_{i}}\big]$
    \STATE Append $\mathcal{T} \leftarrow ~\mathcal{T} \cap (\mathcal{P}_\mathcal{E}, \mathcal{S}_\mathcal{P})$
    \STATE Update $\mathcal{P}_\mathcal{E} \leftarrow \textsl{g}( \mathcal{P}_\mathcal{E} ~|~ \mathcal{D})$
\ENDWHILE
\end{algorithmic}
\end{algorithm}

During the optimization, the trajectory begins from an empty set and starts to incorporate newly derived trigger prompts with the fidelity scores in each optimization iteration. Following this trajectory, the LLM explainer generates a new trigger prompt with the goal of achieving higher fidelity scores of explanations in subsequent iterations. After several rounds of iterations, \Algnameabbr{} ultimately yields an optimal explanation trigger prompt with the highest fidelity score for the LLM explainer to generate a more faithful NL explanation.

\paragraph{Algorithm of Trigger Prompt Optimization.}
The outline of \Algnameabbr{} for Trigger Prompt Optimization is detailed in Algorithm~\ref{alg:xllm-trg},
which focuses on optimizing the trigger prompt for generating NL explanations. Specifically, in each iteration, LLM explainer $\textsl{g}(\cdot)$ leverages the trigger prompt to generate the NL explanations and estimates its fidelity~(lines 4-7). The fidelity scores of the trigger prompts average the fidelity scores of the entire hold-out dataset~(lines 8). Afterward, the trajectory appends the trigger prompt with its corresponding fidelity score~(line 9), and updates the trigger prompt as a new sequence of words to achieve higher fidelity scores (line 10). Through multiple iterations, \Algnameabbr{} progressively guides the trigger prompt to generate explanations with higher fidelity scores, following the reference path established in the trajectory. The iteration process terminates at a predetermined 20 step.

\begin{figure*}[t]
\centering
\subfigure{
\includegraphics[width=0.3\linewidth]{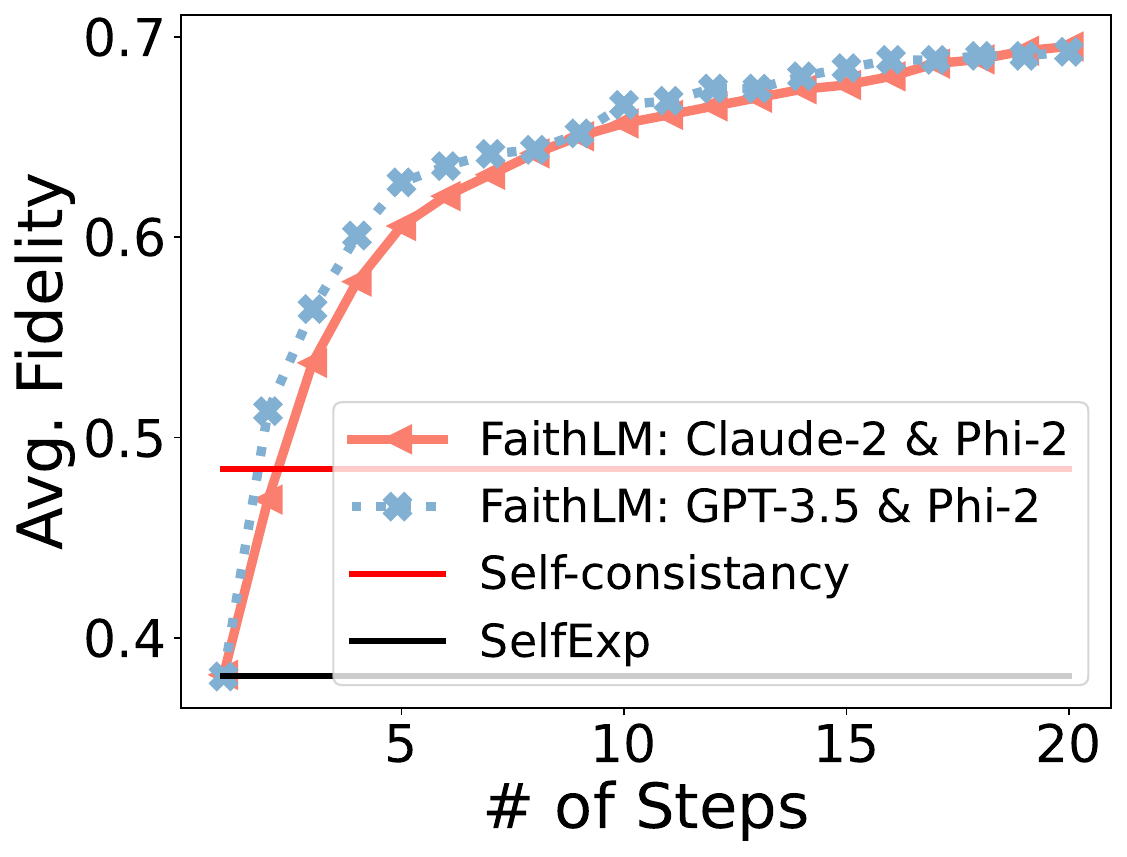}
}
\subfigure{
\includegraphics[width=0.3\linewidth]{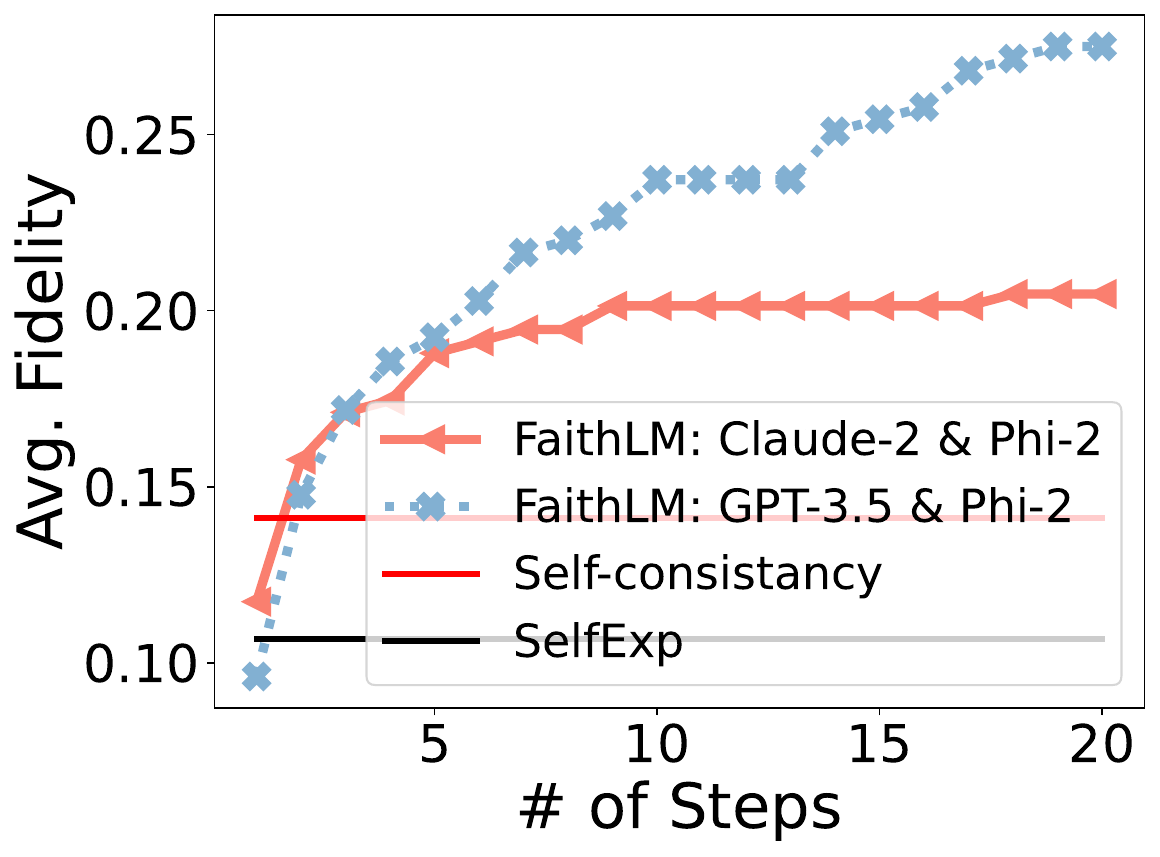}
}
\subfigure{
\includegraphics[width=0.3\linewidth]{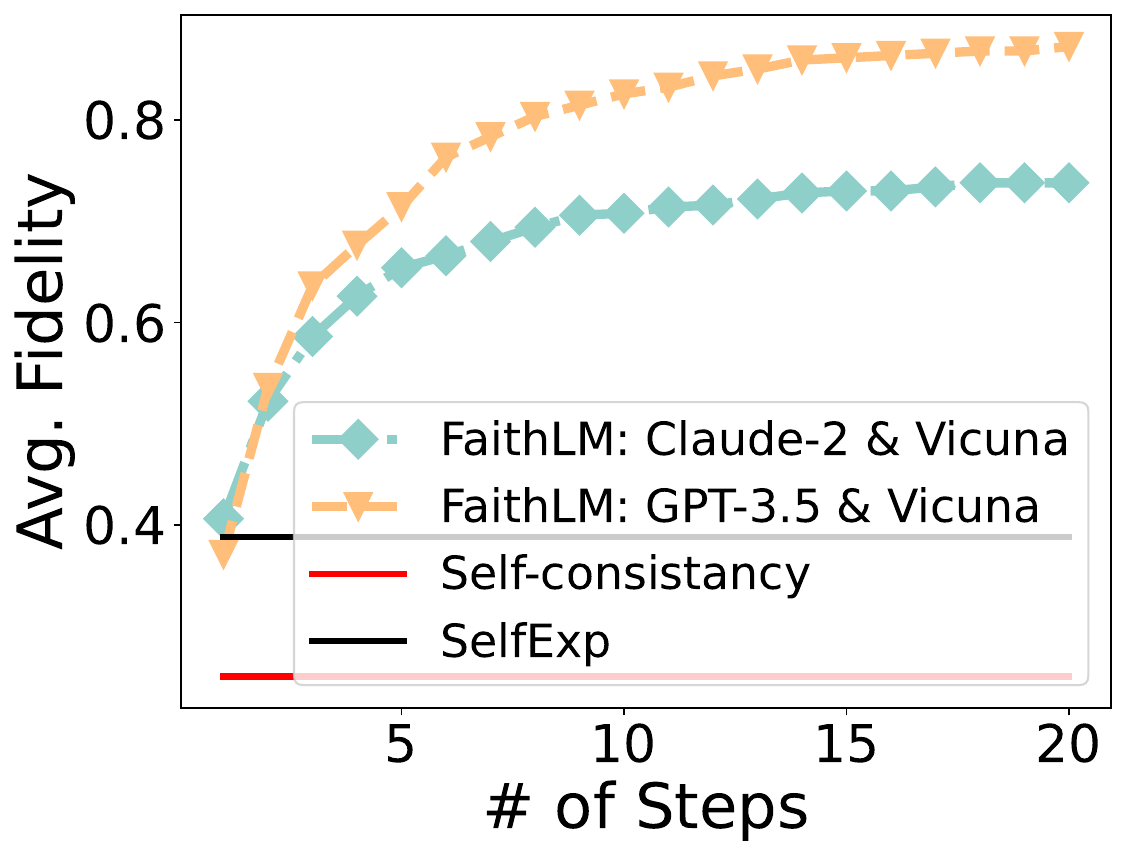}
}
\vspace{-0.4cm}
\caption{The fidelity evaluation of explanations on ECQA (left), TriviaQA-Long (middle), and COPA dataset (right). The reported scores are the average fidelity on testing instances in each step of fidelity-enhanced optimization.}
% \vspace{-0.3cm}
\label{fig:local_exp_accum}
\end{figure*}

\section{Experiment}
In this section, we conduct experiments to evaluate the performance of \Algnameabbr{}, aiming to answer the following three research questions: \textbf{RQ1:} How does \Algnameabbr{} perform in generating explanations in terms of efficacy? \textbf{RQ2:} Can optimized explanation trigger prompts transfer between different datasets? \textbf{RQ3:} Does the configurations of LLMs affect the explanation performance of \Algnameabbr{}?

\subsection{Dataset and Baseline}
\textbf{Datasets.} We evaluate \Algnameabbr{} on three datasets with multiple tasks: ECQA~\cite{aggarwaletal2021ecqa} dataset on commonsense question-answer task, TrivaQA-Long~\cite{bai2023longbench, 2017arXivtriviaqa} dataset on reading comprehension task, and COPA~\cite{kavumba-etal-2019-choosing, roemmele2011choice} dataset on commonsense causal reasoning task. More details of datasets are provided in Appendix~\ref{appendix:baseline_detail}. \textbf{Baseline Methods.} We compare \Algnameabbr{} with two state-of-the-art baseline methods: \texttt{SelfExp}~\cite{madsen2024can} and \texttt{Self-consistency}~\cite{wang2022self}. The former ones instruct LLMs to generate explanations using prompt engineering under single-forward inference, and the later ones leverage the outputs from the chain-of-thought prompting process as the model explanations.
% \textcolor{red}{Put the evaluation metrics here.}

\begin{figure*}[t]
    \small
    \centering
    \includegraphics[width=0.34\linewidth]{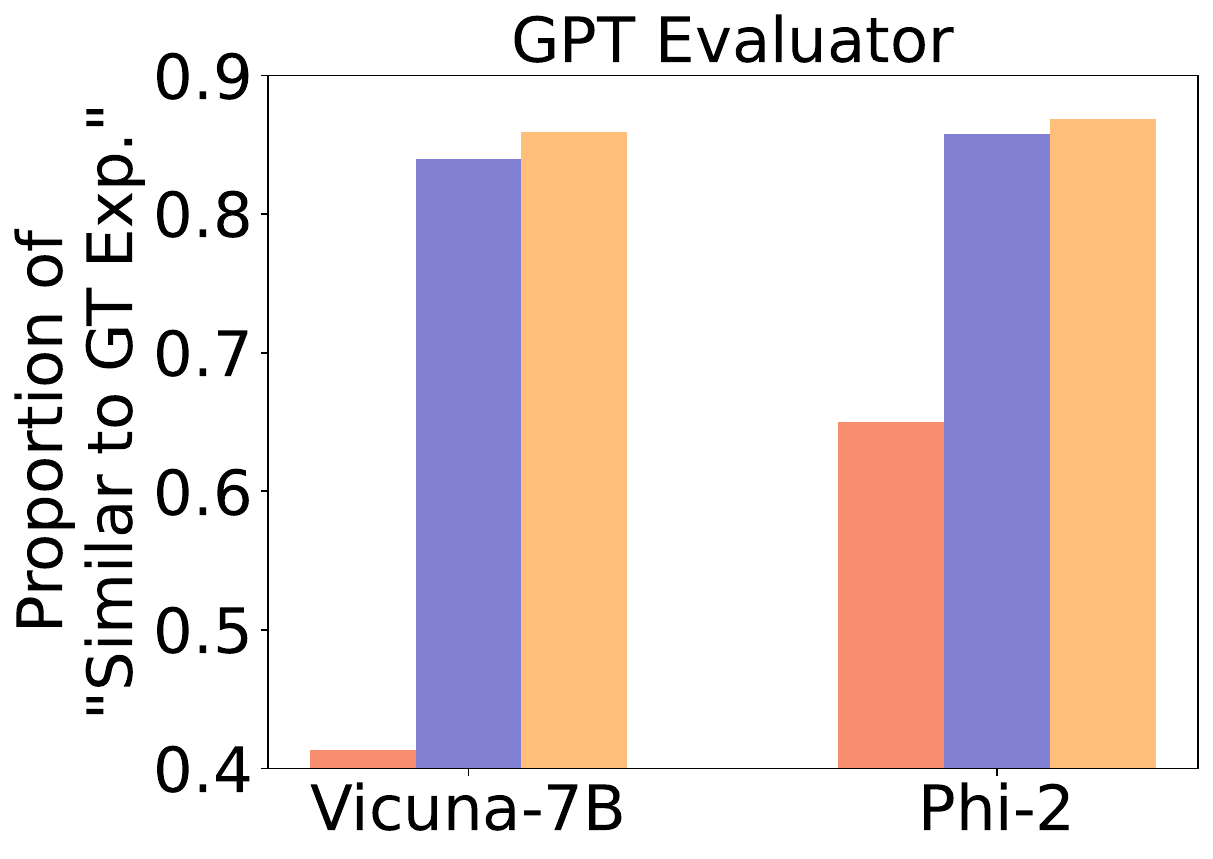}~~
    \includegraphics[width=0.3\linewidth]{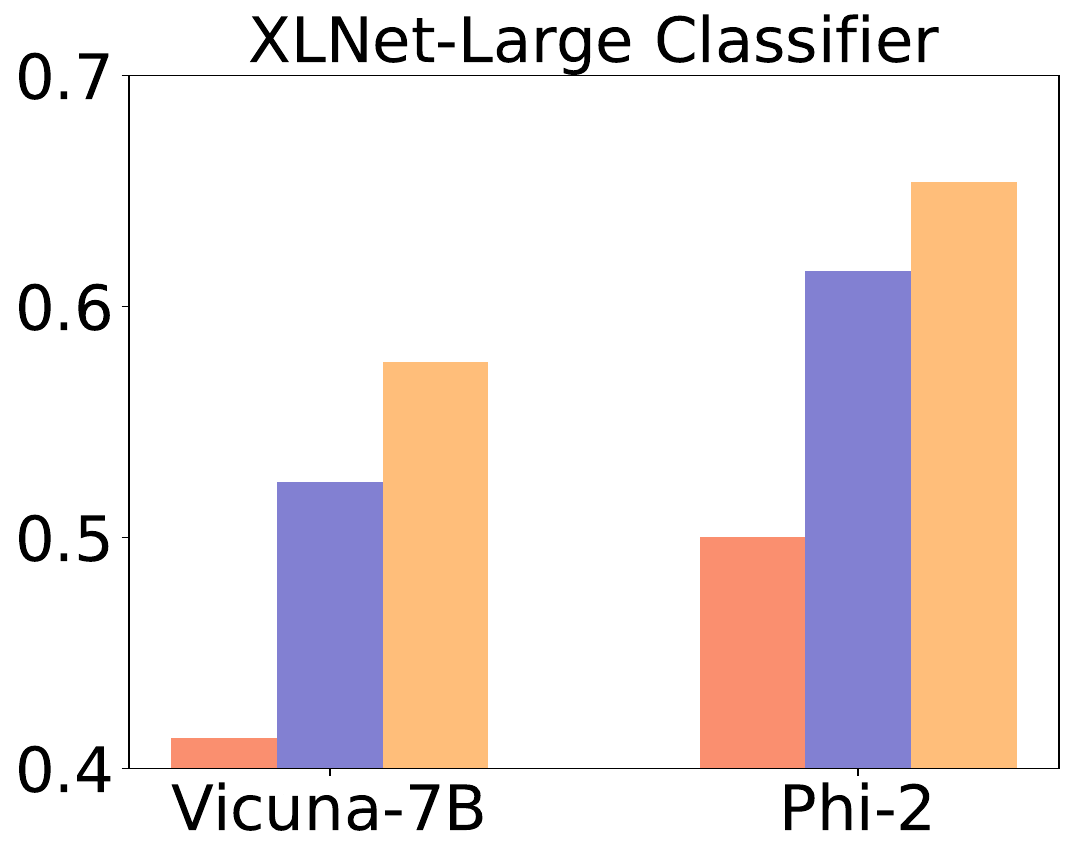}~~
    \includegraphics[width=0.3\linewidth]{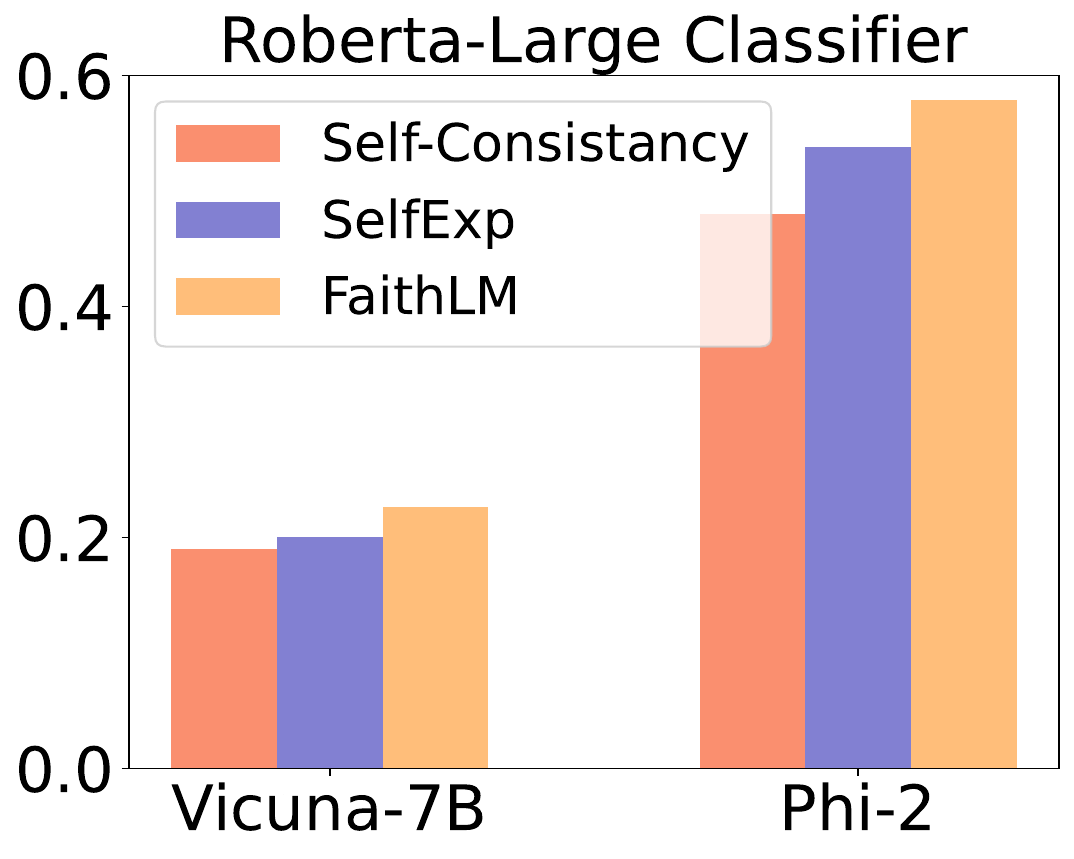}~~
\caption{Trustfulness evaluation of the NL explanations. Higher the proportion of ``similar to ground-truth explanation," the more consistent the derived explanations are with the ground-truth NL explanations.}
% \vspace{0.2cm}
\label{fig:local_exp_gt}
\end{figure*}

\begin{figure*}[t]
\centering
\subfigure{
\includegraphics[width=0.3\linewidth]{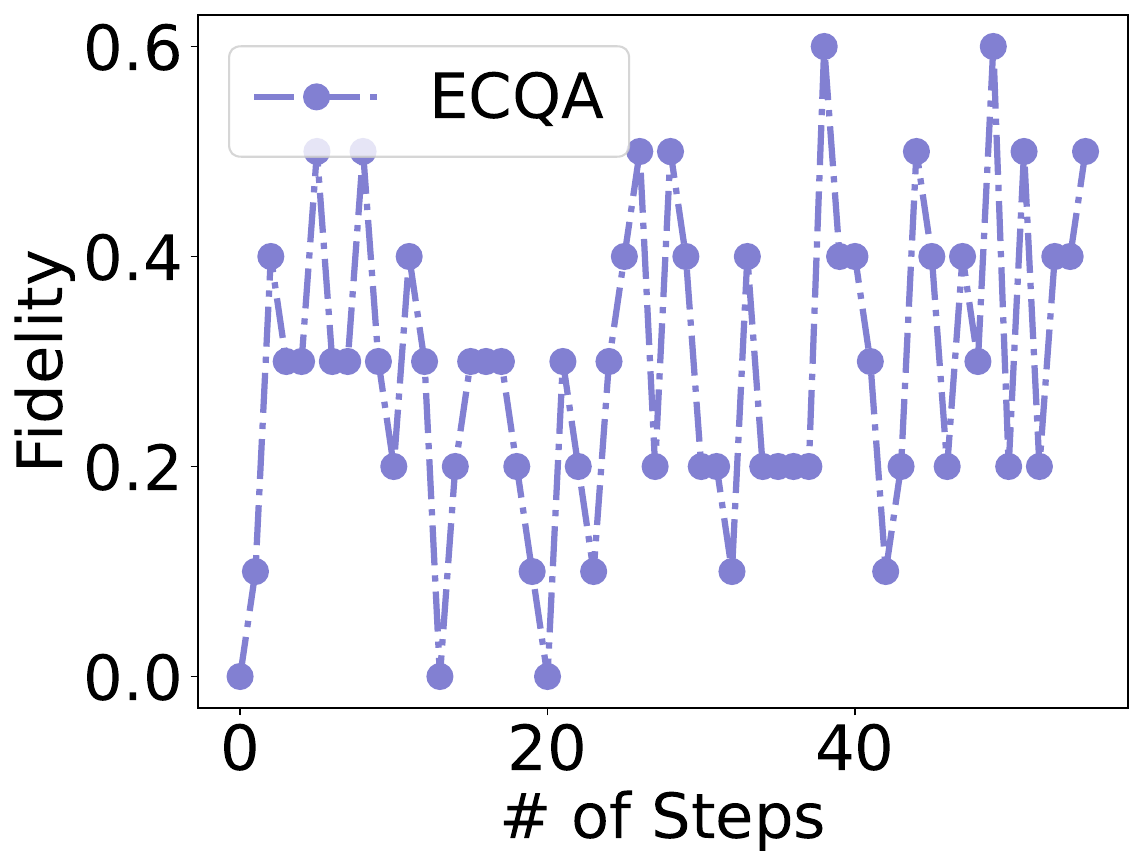}
}~
\subfigure{
\includegraphics[width=0.3\linewidth]{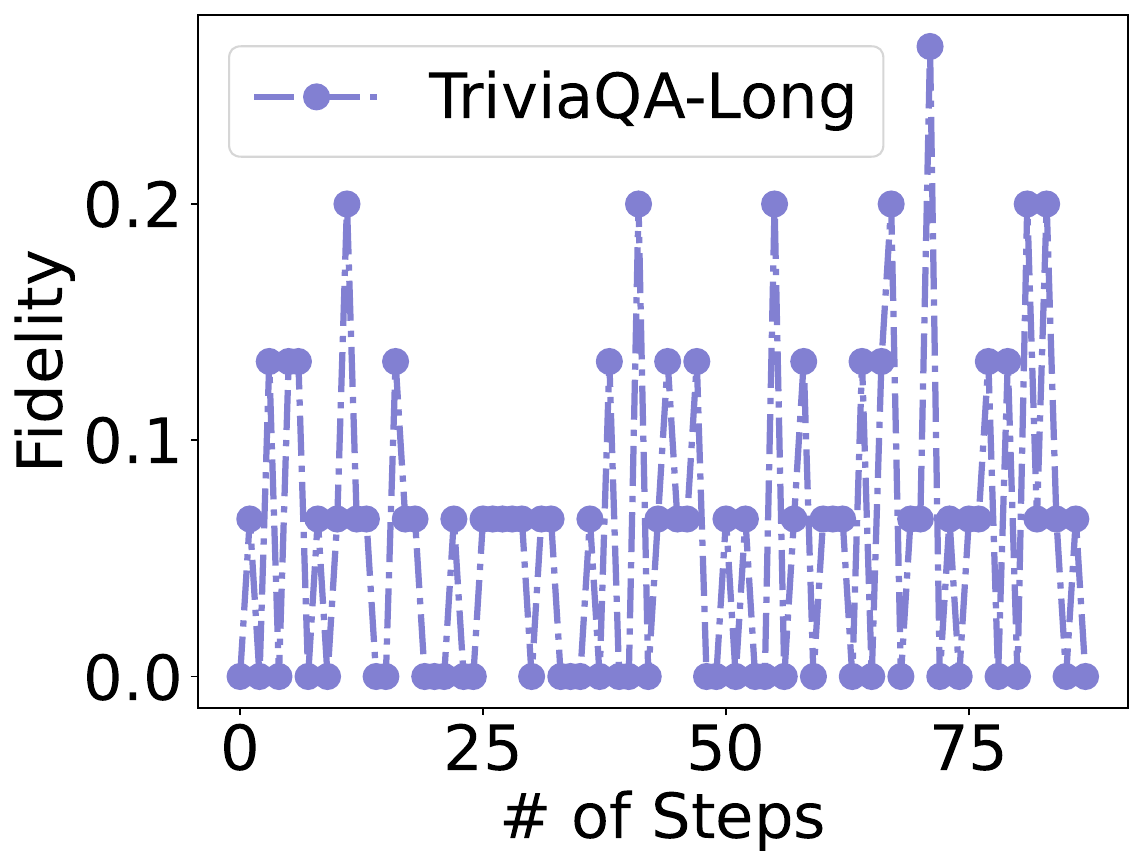}
}~
\subfigure{
\includegraphics[width=0.3\linewidth]{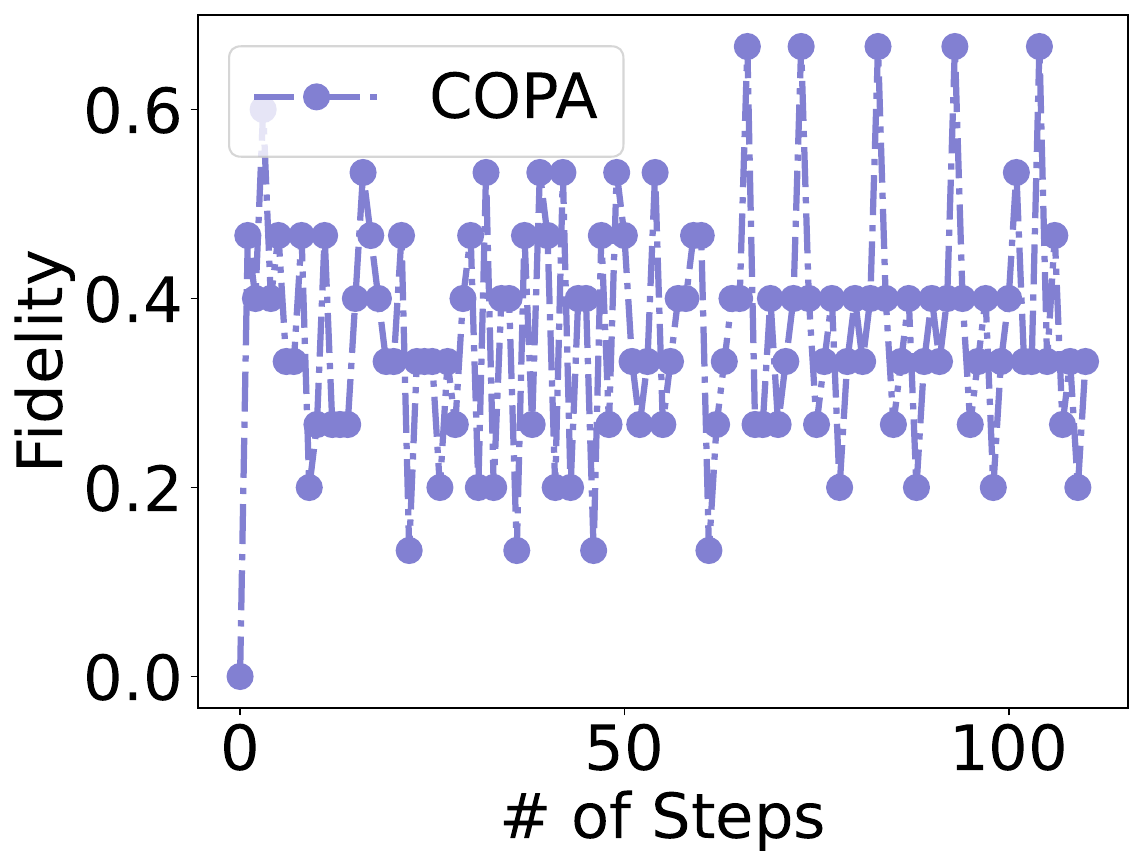}
}
\vspace{-0.4cm}
\caption{The fidelity in different optimization steps of the trigger prompts (Algorithm~\ref{alg:xllm-trg}) on the ECQA, TrivaQA, and COPA datasets. The fidelity grows higher as the number of steps increases.}
% \vspace{-0.2cm}
\label{fig:opt_curve}
\end{figure*}

\paragraph{Comparison with Concurrent Faithfulness Notions.}
Simulatability-based evaluations (e.g., ~\cite{madsen2024can, chen2023algorithms}) assess whether a model can reproduce predictions from an explanation, i.e., behavioral imitation. In contrast, FaithLM measures \emph{causal sensitivity} via a controlled intervention on the explanation’s semantics: we compute $S_E = f(X) - f(X\,|\,\lnot E_{NL})$ and read off prediction shifts induced by a contrary hint, without token masking or task-specific heuristics. This aligns with intervention-driven faithfulness testing advocated by recent work and keeps the target model $f$ unchanged.\footnote{See our definition and evaluator description in \S3.1.}
Empirically, we additionally report a small-scale comparison where we compute (i) simulatability scores and (ii) $S_E$ on the same instances; we analyze where the two agree and where $S_E$ detects causal use beyond imitability.

\subsection{Experimental Settings}
\label{sec:exp}
We evaluate \Algnameabbr{} through two complementary tasks: (i) fidelity-enhanced explanation generation and (ii) explanation-trigger prompt optimization. Details are presented in Appendix~\ref{apdx:exp_set}.

\paragraph{Fidelity-Enhanced Explanation.}
We measure the alignment between generated explanations and predictions~\cite{chen2025reasoning}. For each example, \Algnameabbr{} produces one NL explanation, and the average fidelity score across instances is reported as the final metric.

\paragraph{Explanation-Trigger Prompt Optimization.}
This task optimizes the trigger prompt that guides \Algnameabbr{} to generate more faithful explanations. At each iteration, 30 samples are drawn from the training data as a validation set, and the mean fidelity score on these samples is used as the optimization objective.

\vspace{0.05cm}
\noindent\textbf{Evaluation Metrics.}
We evaluate explanation quality using two metrics: \emph{fidelity} and \emph{truthfulness}.  
Fidelity follows the intervention-based protocol of~\citet{chen2025reasoning}, where contrary hints act as controlled semantic interventions and fidelity is the resulting prediction sensitivity.  
Truthfulness measures semantic consistency between generated and ground-truth explanations using GPT-4o, RoBERTa-Large, and XLNet-Large~\citep{nie-etal-2020-adversarial}, following~\citet{liu2023gpteval}. Evaluators classify each pair as \textit{similar}, \textit{dissimilar}, or \textit{non-relevant}, and the proportion of “similar’’ outputs forms the truthfulness score. Full prompts appear in Appendix~\ref{appendix:prompt_eval}.

\vspace{0.05cm}
\noindent\textbf{Implementation Details.}
We use Vicuna-7B~\citep{vicuna2023} and Phi-2~\citep{phi2} as target models $f(\cdot)$, and GPT-3.5-Turbo and Claude-2~\citep{claude} as explainers $g(\cdot)$. The same models generate contrary hints. Results are averaged over three runs with grid search on hyperparameters. Both Vicuna-7B and Phi-2 share identical decoding settings. Full configurations and infrastructure details are given in Appendices~\ref{appendix:hyp} and \ref{apx:infra}.

\subsection{Explanation Efficacy of \Algnameabbr{} (RQ1)}
\label{sec:local_eff}

% In this section, we first evaluate the efficacy of NL explanations generated by Algorithm~\ref{alg:xllm-exp}, and then assess the explanation trigger prompts generated by Algorithm~\ref{alg:xllm-trg}. 

% \vspace{-0.3cm}
\paragraph{Efficacy of Derived Explanations.}
We assess the efficacy of derived explanations under the fidelity metric. \Algnameabbr{} adopts the trajectory system prompts in Figure~\ref{appendix:fig:exp} of Appendix~\ref{appendix:prompt_xllm}. 
The generation of contrary hints is guided by the prompt in Table~\ref{tab:meta}.
\vspace{0.1cm}
% For the tasks of trigger prompt optimization, it adopts the prompts in Figure~\ref{appendix:fig:trg} of Appendix~\ref{appendix:prompt_xllm}. 
% To showcase the efficacy of the optimization approach in , ... 

\begin{figure*}[t]
\centering
\subfigure[Robustness analytics of trigger prompts]
{
    \includegraphics[width=0.235\linewidth]{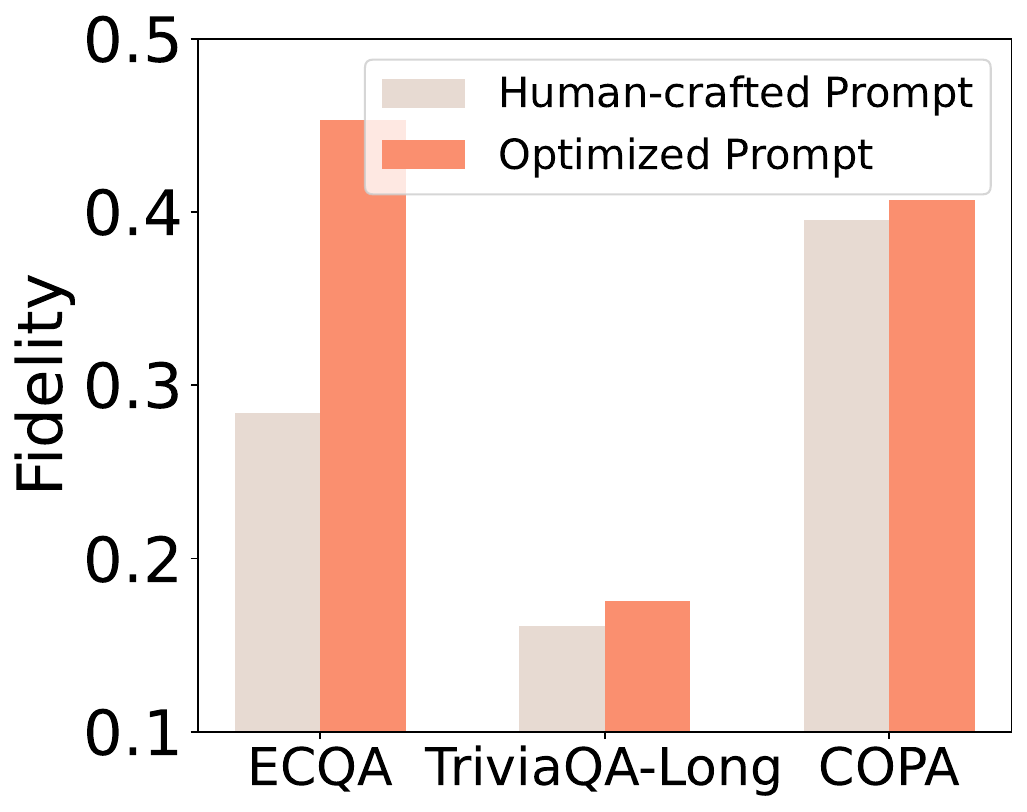}
    \includegraphics[width=0.235\linewidth]{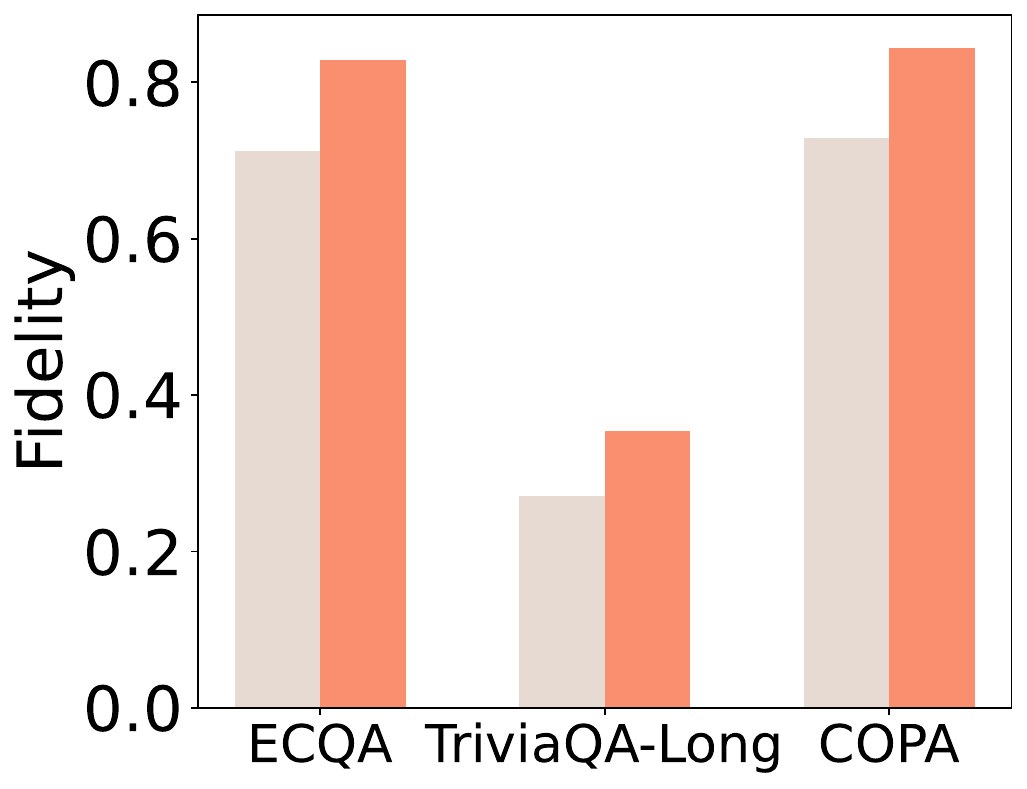}
    \label{fig:opti_trg}
}
\subfigure[Transferability results of trigger prompts]
{
    \includegraphics[width=0.235\linewidth]{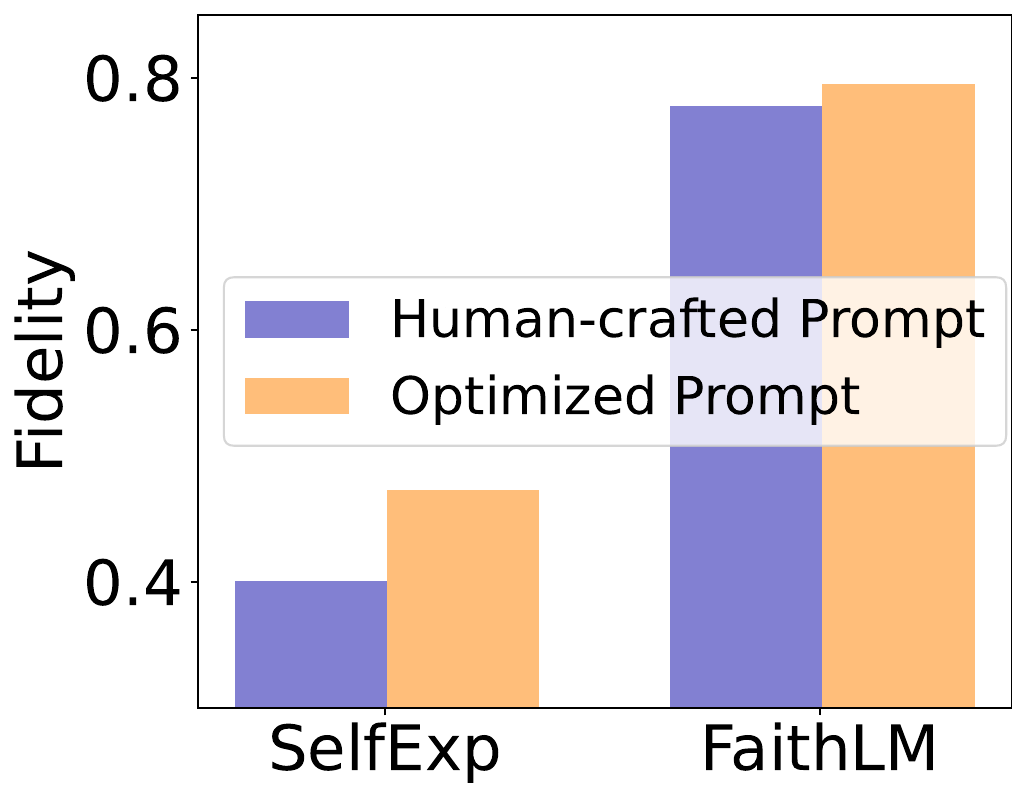}
    \includegraphics[width=0.235\linewidth]{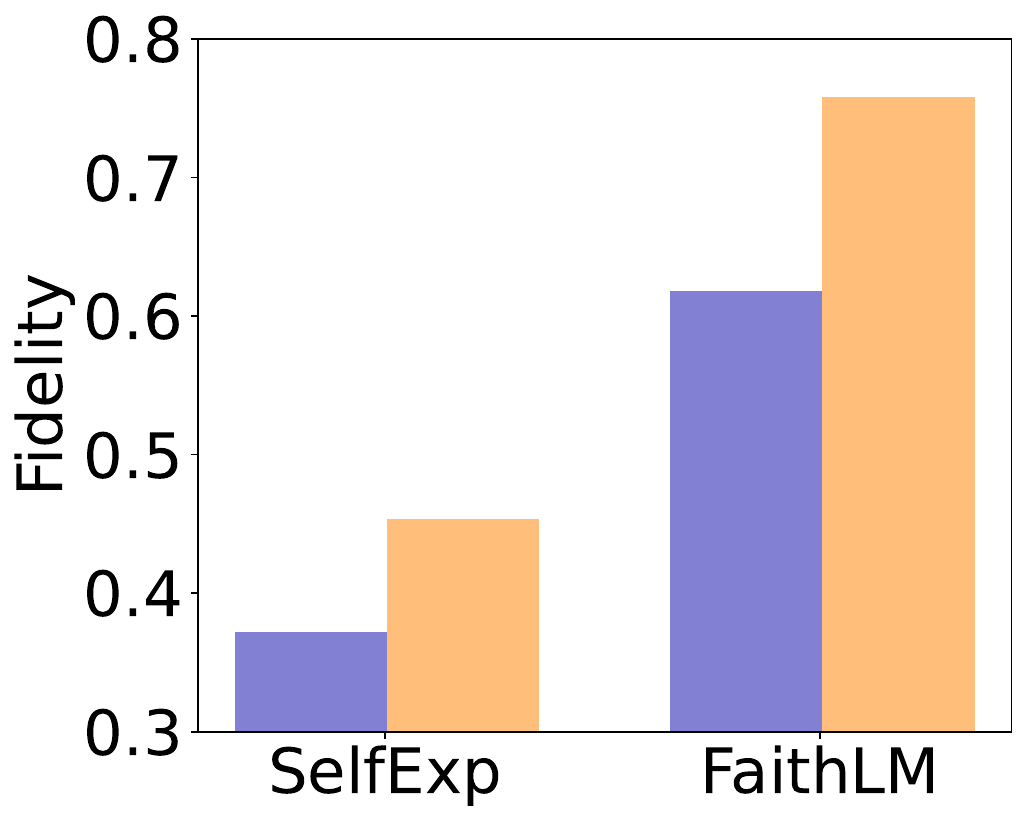}
    \label{fig:trans}
}
\vspace{-0.3cm}
\caption{Assessment on the adaptation of the optimized explanation trigger prompts. Figure (a) reveals the robustness evaluation, and figure (b) illustrates the results on transferability.}
% \vspace{-0.25cm}
\label{fig:nonfac}
\end{figure*}

\begin{figure}[t]
\centering
% \subfigure[Evaluated by NLI classifiers]
% {
%     \includegraphics[width=0.235\linewidth]{figures/noncondi_copa.pdf}
%     \includegraphics[width=0.235\linewidth]{figures/noncondi_ecqa.pdf}
% }
\subfigure[Evaluated by GPT-4o Evaluator]
{
    \includegraphics[width=0.48\linewidth]{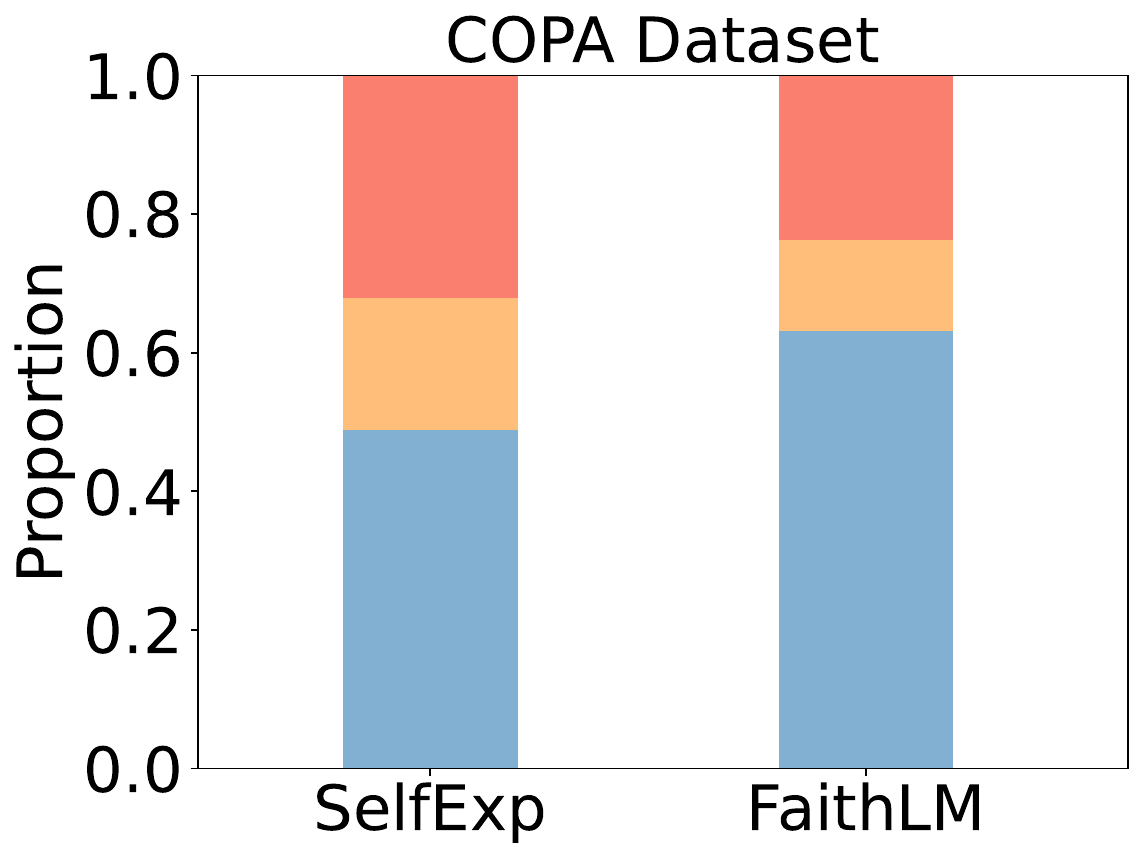}
    \includegraphics[width=0.48\linewidth]{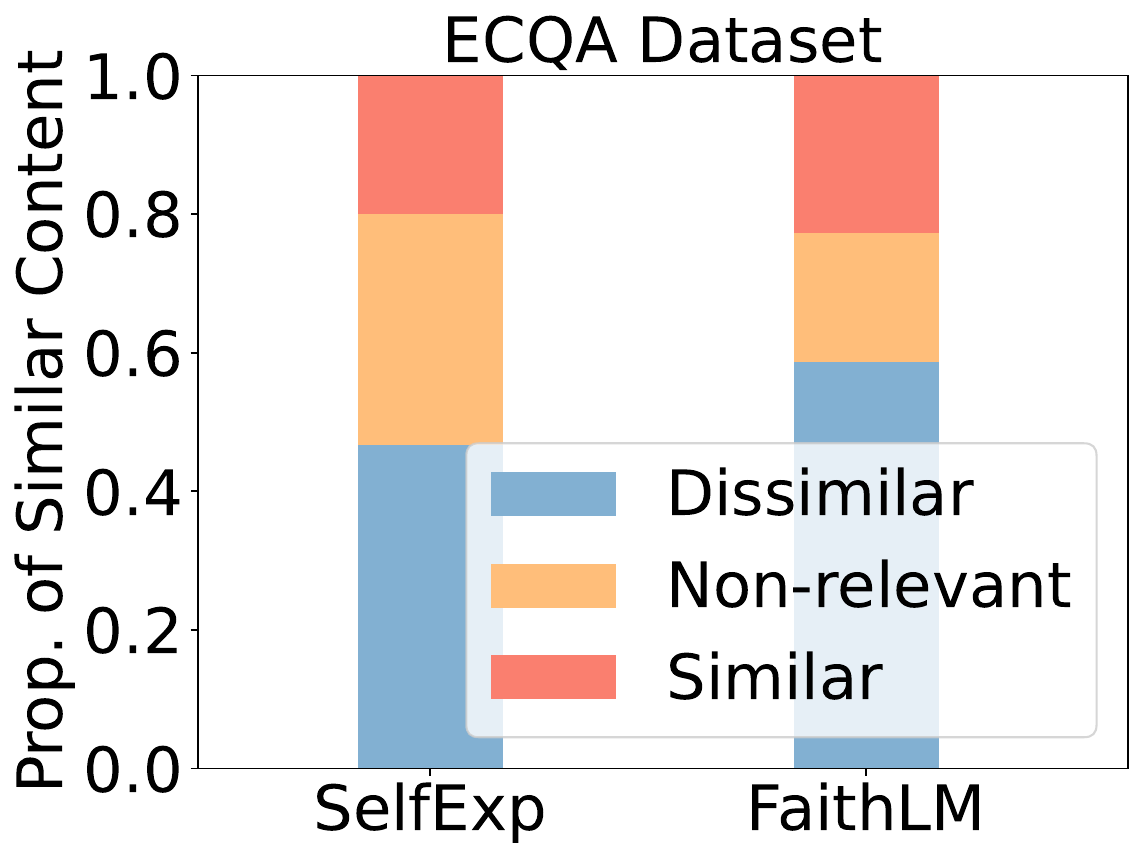}
}
\vspace{-0.25cm}
\caption{Ablation studies on evaluating contrary hint.}
\vspace{-0.25cm}
\label{fig:gt}
\end{figure}

\begin{itemize}[leftmargin=*, topsep=0pt]
    \itemsep=-0.5pt
    \item \textbf{Fidelity Evaluation.} The results in Figure~\ref{fig:local_exp_accum} demonstrate that \Algnameabbr{} achieves significantly higher fidelity scores across all three datasets compared with two baselines after 20 steps of optimization. Moreover, the optimization curve of fidelity demonstrates that 20 rounds of optimization are sufficient to converge. A similar phenomenon occurs across different settings of explainers and targeted LLMs. Additional results are provided in Appendix~\ref{appendix:trg}.
    \item \textbf{Truthfulness Evaluation.} To evaluate the truthfulness of explanations, we show the proportions of ``similar to ground-truth explanations" in the ECQA dataset, as depicted in Figure~\ref{fig:local_exp_gt}. 
    We leverage GPT evaluators and well-trained NLI evaluators to assess whether the given explanations are within similar content to ground-truth explanations. 
    The results show that \Algnameabbr{}'s explanations are more consistent with the ground-truth NL explanations, indicated by a larger proportion of ``similar to ground-truth explanations" generated by \Algnameabbr{} than baseline methods. 
    
\end{itemize}

% \vspace{-0.25cm}
\vspace{0.1cm}
\noindent\textbf{Efficacy of Explanation Trigger Prompts.}
We first show prompt optimization curves on three different datasets, and then leverage the optimal explanation trigger prompts to generate explanations via \Algnameabbr{}. In the experiments, we randomly select 15 instances from the training dataset in each optimization round, and compute the average fidelity scores of the newly derived trigger prompts. After the progress is terminated, we evaluate the optimized trigger prompts on the testing set. The optimization step is uniformly established at 50 rounds across different explainer and targeted LLMs.

\begin{itemize}[leftmargin=*]
    \itemsep=-0.5pt
    \item \textbf{Trigger Prompt Optimization Curve.} Figure~\ref{fig:opt_curve} demonstrates the optimization curves of three datasets. We display the explainer as GPT-3.5-Turbo and Claude-2 and the explainer as Vicuna-7B. We observe that the optimization curve exhibits a generally ascending trend as the step progresses, interspersed with multiple waves throughout the optimization procedure. This indicates that \Algnameabbr{} generates better explanation trigger prompts after the optimization. More results of optimization curves on remaining datasets are provided in Appendix~\ref{appendix:trg}.
    
    \item \textbf{Explanation Generation by Optimized Trigger Prompts.}  We utilize the optimized explanation trigger prompts to generate explanations following Algorithm~\ref{alg:xllm-exp}. The results are displayed in Figure~\ref{fig:opti_trg}, including the experiments conducted using all three datasets with Claude-2 as the explainer and Vicuna-7B as the targeted LLM. We observe that optimized explanation trigger prompts obtain higher fidelity scores than the initial human-crafted trigger prompt in generating explanations. This trend is consistent across all datasets, regardless of whether the explanations are refined by Algorithm~\ref{alg:xllm-exp}. 
    % Additionally, the optimized trigger prompts are effective in improving the fidelity.
\end{itemize}

% \vspace{0.1cm}
% \noindent\textbf{A Case Study of ~\Algnameabbr{}.} The case studies illustrate the evolving trend via \Algnameabbr{}, including derived NL explanations, explanation trigger prompts, and contrary hints in Appendix~\ref{appendix:case}. These studies demonstrate that the explanations are informative and readable, enabling humans to understand the reasons behind the decision-making process of target LLMs.
% and human-understandable, suggesting that the decision-making process of LLMs are truly reflected.

\subsection{Transferability of Trigger Prompt (RQ2)}
We assess the transferability of ultimately optimized trigger prompts across different unseen datasets within the same domain, as depicted in Figure~\ref{fig:trans}. Specifically, we transfer the optimized trigger prompts from the ECQA to the Social-IQA dataset, and from the COPA to the XCOPA datasets, without any additional optimization. Specifically, the Social-IQA dataset is dedicated to commonsense question-answering (similar to the ECQA dataset), while the XCOPA dataset specializes in causal reasoning (similar to the COPA dataset). We adopt the Vicuna-7B as the targeted LLM, and Claude-2 as the explainer on these transfer tasks. The fidelity of the derived NL explanation on the target dataset is shown in Figure~\ref{fig:nonfac}(b). The optimized trigger prompts show better explanation efficacy than human-crafted prompts when it is transferred in similar domain. This shows that the optimized trigger prompts generated by \Algnameabbr{} possess a great data transferability.

\subsection{Ablation Studies on contrary hint (RQ3)}
\label{sec:abl_non}
The quality of contrary hints $\neg\mathcal{E}_\text{NL}$ determines the efficacy of \Algnameabbr{}.
We leverage the powerful LLMs as the LLM agent to generate contrary hints, requesting the delivery of high-quality opposite-meaning outputs from their original NL explanations. In this section, we evaluate the quality of contrary hints, aiming to observe the semantic differences between the original NL explanations and their contrary hints. 
To examine the quality, we employ one GPT-4o classifier and two well-trained NLI classifiers, Roberta-Large and XLNet-Large~\cite{nie-etal-2020-adversarial}. We leverage each classifier to distinguish whether the relationship between the ``original NL explanations" and ``contrary hints" belong to the category of ``similar meaning (entailment)," ``dissimilar meaning (contradiction)," or ``non-relevant (neutral)." We follow the evaluation settings from~\cite{liu2023gpteval} on GPT-classifier with evaluation prompt provided in Table~\ref{tab:eval}.

The results are shown in Figure~\ref{fig:gt} and Figure~\ref{fig:cont_class} under the randomly sampled 100 instances from the ECQA and COPA datasets. We observe that the two NLI classifiers achieve up to 86\% and 82\% in the ``dissimilar meanings"" category on the ECQA and COPA datasets, respectively. The results of the GPT-classifier demonstrate that the derived explanations from SelfExp obtain more non-faithful information than \Algnameabbr{}, risking the LLM agent of Fidelity Evaluator in generating non-relevant information as the contrary hints. Case studies are provided in Appendix~\ref{appendix:case} to show the informativeness and readability of contrary hints.

% \subsection{Explanation Truthfulness Evaluation (RQ3)}
% We evaluate all baseline methods and \Algnameabbr{} under multiple settings to assess how closely the derived explanations match the ground-truth rationales. A GPT-based evaluator assigns a GPT-Score from 1 to 5, with higher values indicating greater semantic similarity. Explanations tagged as “similar content” or scoring near 5 are treated as matches. As shown in Figure~\ref{fig:local_exp_gt_appendix}, our method yields the highest proportion of such matches, indicating that the derived explanations closely align with the ground-truth explanations.

% \begin{figure}[h]
%     \centering
%     \includegraphics[width=0.35\linewidth]{figures/gt_exp_local_score.pdf}
%     \caption{Truthfulness evaluation with ground-truth explanation under GPT-score settings.}
% \vspace{-0.35cm}
% \label{fig:local_exp_gt_appendix}
% \end{figure}

\section{Conclusion}

In this paper, we introduce \Algnameabbr{} to explain the decision-making process of LLMs, instead of providing reasoning or self-refinement feedback as model explanation. Specifically, \Algnameabbr{} employs a fidelity enhancement strategy to progressively refine the fidelity of derived explanations and explanation trigger prompts. \Algnameabbr{} conducts an iterative process to improve the fidelity of derived explanations. Theorem~\ref{thm:intervention} establishes the contrary-hint score as a valid measure of faithfulness.
Experimental results demonstrate the effectiveness of \Algnameabbr{}, and better alignment with the ground-truth explanations. This suggest that the decision-making process are truly reflected. For future work, we plan to extend \Algnameabbr{} in healthcare, where the needs for transparency is critical given the growing reliance on black-box LLMs. 

% In this work , we propose \Algnameabbr{}, a framework for explaining LLM decisions through causal fidelity rather than heuristic reasoning. \Algnameabbr{} iteratively refines explanations and trigger prompts via contrary-hint interventions that test whether explanation content is causally used in predictions. Theorem~\ref{thm:intervention} establishes the contrary-hint score as a valid measure of faithfulness. Empirical results across three datasets demonstrate that \Algnameabbr{} consistently produces explanations with higher fidelity and stronger semantic alignment with human-annotated ground truth, indicating that the derived explanations accurately reflect the model’s internal decision process. Future work will extend \Algnameabbr{} to transparency-critical domains such as healthcare.

\clearpage
\section{Limitations}
\label{apdx:limit}
\textcolor{black}{
One significant limitation of \Algnameabbr{} associated with the carbon emissions during the experiments. To generate better fidelity and truthfulness of the derived explanations, \Algnameabbr{} requires iterating a few rounds of optimization during the explanation derivation. This leads to extra computational resources to proceed. The extra computational power is thus required for optimizing \Algnameabbr{} leads to considerable energy consumption, which, in turn, results in a significant carbon footprint. As the demand for more sophisticated LLMs continues to grow, so does their environmental impact. This limitation underscores the urgent need to explore and adopt more sustainable practices and technologies in the development and learning \Algnameabbr{} with fewer optimization steps to mitigate their ecological footprint. 
}
\vspace{-0.5cm}

\bibliography{paper}

\begin{thebibliography}{51}
\providecommand{\natexlab}[1]{#1}

\bibitem[{Achiam et~al.(2023)Achiam, Adler, Agarwal, Ahmad, Akkaya, Aleman, Almeida, Altenschmidt, Altman, Anadkat et~al.}]{achiam2023gpt}
Josh Achiam, Steven Adler, Sandhini Agarwal, Lama Ahmad, Ilge Akkaya, Florencia~Leoni Aleman, Diogo Almeida, Janko Altenschmidt, Sam Altman, Shyamal Anadkat, et~al. 2023.
\newblock Gpt-4 technical report.
\newblock \emph{arXiv preprint arXiv:2303.08774}.

\bibitem[{Aggarwal et~al.(2021)Aggarwal, Mandowara, Agrawal, Khandelwal, Singla, and Garg}]{aggarwaletal2021ecqa}
Shourya Aggarwal, Divyanshu Mandowara, Vishwajeet Agrawal, Dinesh Khandelwal, Parag Singla, and Dinesh Garg. 2021.
\newblock {E}xplanations for {C}ommonsense{QA}: {N}ew {D}ataset and {M}odels.
\newblock In \emph{Proceedings of the 59th Annual Meeting of the Association for Computational Linguistics and the 11th International Joint Conference on Natural Language Processing}, page 3050–3065. Association for Computational Linguistics.

\bibitem[{Anthropic(2023)}]{claude}
Anthropic. 2023.
\newblock \href {https://www.anthropic.com/index/claude-2} {Claude 2}.

\bibitem[{AnthropicAI(2023)}]{claude2023}
AnthropicAI. 2023.
\newblock Introducing claude.
\newblock \emph{See https://www.anthropic.com/index/introducing-claude}.

\bibitem[{Bai et~al.(2023)Bai, Lv, Zhang, Lyu, Tang, Huang, Du, Liu, Zeng, Hou, Dong, Tang, and Li}]{bai2023longbench}
Yushi Bai, Xin Lv, Jiajie Zhang, Hongchang Lyu, Jiankai Tang, Zhidian Huang, Zhengxiao Du, Xiao Liu, Aohan Zeng, Lei Hou, Yuxiao Dong, Jie Tang, and Juanzi Li. 2023.
\newblock \href {https://arxiv.org/abs/2308.14508} {Longbench: A bilingual, multitask benchmark for long context understanding}.
\newblock \emph{Preprint}, arXiv:2308.14508.

\bibitem[{Chen et~al.(2023{\natexlab{a}})Chen, Covert, Lundberg, and Lee}]{chen2023algorithms}
Hugh Chen, Ian~C Covert, Scott~M Lundberg, and Su-In Lee. 2023{\natexlab{a}}.
\newblock Algorithms to estimate shapley value feature attributions.
\newblock \emph{Nature Machine Intelligence}, pages 1--12.

\bibitem[{Chen et~al.(2021)Chen, Ji, Zeng, Li, Zhang, Chen, and Zhang}]{chen2021kace}
Qianglong Chen, Feng Ji, Xiangji Zeng, Feng-Lin Li, Ji~Zhang, Haiqing Chen, and Yin Zhang. 2021.
\newblock Kace: Generating knowledge aware contrastive explanations for natural language inference.
\newblock In \emph{Proceedings of the 59th Annual Meeting of the Association for Computational Linguistics and the 11th International Joint Conference on Natural Language Processing (Volume 1: Long Papers)}, pages 2516--2527.

\bibitem[{Chen et~al.(2025)Chen, Benton, Radhakrishnan, Uesato, Denison, Schulman, Somani, Hase, Wagner, Roger et~al.}]{chen2025reasoning}
Yanda Chen, Joe Benton, Ansh Radhakrishnan, Jonathan Uesato, Carson Denison, John Schulman, Arushi Somani, Peter Hase, Misha Wagner, Fabien Roger, et~al. 2025.
\newblock Reasoning models don't always say what they think.
\newblock \emph{arXiv preprint arXiv:2505.05410}.

\bibitem[{Chen et~al.(2023{\natexlab{b}})Chen, Zhong, Ri, Zhao, He, Steinhardt, Yu, and McKeown}]{chen2023models}
Yanda Chen, Ruiqi Zhong, Narutatsu Ri, Chen Zhao, He~He, Jacob Steinhardt, Zhou Yu, and Kathleen McKeown. 2023{\natexlab{b}}.
\newblock Do models explain themselves? counterfactual simulatability of natural language explanations.
\newblock \emph{arXiv preprint arXiv:2307.08678}.

\bibitem[{Chen et~al.(2023{\natexlab{c}})Chen, Singh, and Sra}]{chen2023lmexplainer}
Zichen Chen, Ambuj~K Singh, and Misha Sra. 2023{\natexlab{c}}.
\newblock Lmexplainer: a knowledge-enhanced explainer for language models.
\newblock \emph{arXiv preprint arXiv:2303.16537}.

\bibitem[{Cheng et~al.(2023)Cheng, Liu, Zheng, Ke, Wang, Dong, Tang, and Huang}]{cheng2023black}
Jiale Cheng, Xiao Liu, Kehan Zheng, Pei Ke, Hongning Wang, Yuxiao Dong, Jie Tang, and Minlie Huang. 2023.
\newblock Black-box prompt optimization: Aligning large language models without model training.
\newblock \emph{arXiv preprint arXiv:2311.04155}.

\bibitem[{Chiang et~al.(2023)Chiang, Li, Lin, Sheng, Wu, Zhang, Zheng, Zhuang, Zhuang, Gonzalez, Stoica, and Xing}]{vicuna2023}
Wei-Lin Chiang, Zhuohan Li, Zi~Lin, Ying Sheng, Zhanghao Wu, Hao Zhang, Lianmin Zheng, Siyuan Zhuang, Yonghao Zhuang, Joseph~E. Gonzalez, Ion Stoica, and Eric~P. Xing. 2023.
\newblock \href {https://lmsys.org/blog/2023-03-30-vicuna/} {Vicuna: An open-source chatbot impressing gpt-4 with 90\%* chatgpt quality}.

\bibitem[{Chuang et~al.(2023)Chuang, Wang, Yang, Liu, Cai, Du, and Hu}]{chuang2023efficient}
Yu-Neng Chuang, Guanchu Wang, Fan Yang, Zirui Liu, Xuanting Cai, Mengnan Du, and Xia Hu. 2023.
\newblock Efficient xai techniques: A taxonomic survey.
\newblock \emph{arXiv preprint arXiv:2302.03225}.

\bibitem[{Du et~al.(2019)Du, Liu, and Hu}]{du2019techniques}
Mengnan Du, Ninghao Liu, and Xia Hu. 2019.
\newblock Techniques for interpretable machine learning.
\newblock \emph{Communications of the ACM}, 63(1):68--77.

\bibitem[{Floridi(2019)}]{floridi2019establishing}
Luciano Floridi. 2019.
\newblock Establishing the rules for building trustworthy ai.
\newblock \emph{Nature Machine Intelligence}, 1(6):261--262.

\bibitem[{Goodman et~al.(2017)Goodman, Flaxman, and X}]{goodman2017european}
Bryce Goodman, Seth Flaxman, and Y~X. 2017.
\newblock European union regulations on algorithmic decision-making and a “right to explanation”.
\newblock \emph{AI magazine}, 38(3):50--57.

\bibitem[{Guo et~al.(2023)Guo, Wang, Guo, Li, Song, Tan, Liu, Bian, and Yang}]{guo2023connecting}
Qingyan Guo, Rui Wang, Junliang Guo, Bei Li, Kaitao Song, Xu~Tan, Guoqing Liu, Jiang Bian, and Yujiu Yang. 2023.
\newblock Connecting large language models with evolutionary algorithms yields powerful prompt optimizers.
\newblock \emph{arXiv preprint arXiv:2309.08532}.

\bibitem[{Huang et~al.(2023)Huang, Mamidanna, Jangam, Zhou, and Gilpin}]{huang2023can}
Shiyuan Huang, Siddarth Mamidanna, Shreedhar Jangam, Yilun Zhou, and Leilani~H Gilpin. 2023.
\newblock Can large language models explain themselves? a study of llm-generated self-explanations.
\newblock \emph{arXiv preprint arXiv:2310.11207}.

\bibitem[{Javaheripi and Bubeck(2023)}]{phi2}
Mojan Javaheripi and Sébastien Bubeck. 2023.
\newblock Phi-2: The surprising power of small language models.

\bibitem[{{Joshi} et~al.(2017){Joshi}, {Choi}, {Weld}, and {Zettlemoyer}}]{2017arXivtriviaqa}
Mandar {Joshi}, Eunsol {Choi}, Daniel {Weld}, and Luke {Zettlemoyer}. 2017.
\newblock \href {https://arxiv.org/abs/1705.03551} {{triviaqa: A Large Scale Distantly Supervised Challenge Dataset for Reading Comprehension}}.
\newblock \emph{arXiv e-prints}, arXiv:1705.03551.

\bibitem[{Kavumba et~al.(2019)Kavumba, Inoue, Heinzerling, Singh, Reisert, and Inui}]{kavumba-etal-2019-choosing}
Pride Kavumba, Naoya Inoue, Benjamin Heinzerling, Keshav Singh, Paul Reisert, and Kentaro Inui. 2019.
\newblock \href {https://doi.org/10.18653/v1/D19-6004} {When choosing plausible alternatives, clever hans can be clever}.
\newblock In \emph{Proceedings of the First Workshop on Commonsense Inference in Natural Language Processing}, pages 33--42, Hong Kong, China. Association for Computational Linguistics.

\bibitem[{Kumar and Talukdar(2020)}]{kumar2020nile}
Sawan Kumar and Partha Talukdar. 2020.
\newblock Nile: Natural language inference with faithful natural language explanations.
\newblock \emph{arXiv preprint arXiv:2005.12116}.

\bibitem[{Lanham et~al.(2023)Lanham, Chen, Radhakrishnan, Steiner, Denison, Hernandez, Li, Durmus, Hubinger, Kernion et~al.}]{lanham2023measuring}
Tamera Lanham, Anna Chen, Ansh Radhakrishnan, Benoit Steiner, Carson Denison, Danny Hernandez, Dustin Li, Esin Durmus, Evan Hubinger, Jackson Kernion, et~al. 2023.
\newblock Measuring faithfulness in chain-of-thought reasoning.
\newblock \emph{arXiv preprint arXiv:2307.13702}.

\bibitem[{Lightman et~al.(2023)Lightman, Kosaraju, Burda, Edwards, Baker, Lee, Leike, Schulman, Sutskever, and Cobbe}]{lightman2023let}
Hunter Lightman, Vineet Kosaraju, Yura Burda, Harri Edwards, Bowen Baker, Teddy Lee, Jan Leike, John Schulman, Ilya Sutskever, and Karl Cobbe. 2023.
\newblock Let's verify step by step.
\newblock \emph{arXiv preprint arXiv:2305.20050}.

\bibitem[{Liu et~al.(2023)Liu, Iter, Xu, Wang, Xu, and Zhu}]{liu2023gpteval}
Yang Liu, Dan Iter, Yichong Xu, Shuohang Wang, Ruochen Xu, and Chenguang Zhu. 2023.
\newblock Gpteval: Nlg evaluation using gpt-4 with better human alignment.
\newblock \emph{arXiv preprint arXiv:2303.16634}.

\bibitem[{Lopardo et~al.(2023)Lopardo, Precioso, and Garreau}]{lopardo2023faithful}
Gianluigi Lopardo, Frederic Precioso, and Damien Garreau. 2023.
\newblock Faithful and robust local interpretability for textual predictions.
\newblock \emph{arXiv preprint arXiv:2311.01605}.

\bibitem[{Lyu et~al.(2023)Lyu, Havaldar, Stein, Zhang, Rao, Wong, Apidianaki, and Callison-Burch}]{lyu2023faithful}
Qing Lyu, Shreya Havaldar, Adam Stein, Li~Zhang, Delip Rao, Eric Wong, Marianna Apidianaki, and Chris Callison-Burch. 2023.
\newblock Faithful chain-of-thought reasoning.
\newblock \emph{arXiv preprint arXiv:2301.13379}.

\bibitem[{Madaan et~al.(2024)Madaan, Tandon, Gupta, Hallinan, Gao, Wiegreffe, Alon, Dziri, Prabhumoye, Yang et~al.}]{madaan2024self}
Aman Madaan, Niket Tandon, Prakhar Gupta, Skyler Hallinan, Luyu Gao, Sarah Wiegreffe, Uri Alon, Nouha Dziri, Shrimai Prabhumoye, Yiming Yang, et~al. 2024.
\newblock Self-refine: Iterative refinement with self-feedback.
\newblock \emph{Advances in Neural Information Processing Systems}, 36.

\bibitem[{Madsen et~al.(2024)Madsen, Chandar, and Reddy}]{madsen2024can}
Andreas Madsen, Sarath Chandar, and Siva Reddy. 2024.
\newblock \href {https://doi.org/10.18653/v1/2024.findings-acl.19} {Are self-explanations from large language models faithful?}
\newblock In \emph{Findings of the Association for Computational Linguistics: ACL 2024}, pages 295--337, Bangkok, Thailand. Association for Computational Linguistics.

\bibitem[{Majumder et~al.(2021)Majumder, Camburu, Lukasiewicz, and McAuley}]{majumder2021knowledge}
Bodhisattwa~Prasad Majumder, Oana-Maria Camburu, Thomas Lukasiewicz, and Julian McAuley. 2021.
\newblock Knowledge-grounded self-rationalization via extractive and natural language explanations.
\newblock \emph{arXiv preprint arXiv:2106.13876}.

\bibitem[{Menon et~al.(2023)Menon, Zaman, and Srivastava}]{menon2023mantle}
Rakesh~R Menon, Kerem Zaman, and Shashank Srivastava. 2023.
\newblock Mantle: Model-agnostic natural language explainer.
\newblock \emph{arXiv preprint arXiv:2305.12995}.

\bibitem[{Mir{\'o}-Nicolau et~al.(2024)Mir{\'o}-Nicolau, Jaume-i Cap{\'o}, and Moy{\`a}-Alcover}]{miro2024comprehensive}
Miquel Mir{\'o}-Nicolau, Antoni Jaume-i Cap{\'o}, and Gabriel Moy{\`a}-Alcover. 2024.
\newblock A comprehensive study on fidelity metrics for xai.
\newblock \emph{arXiv preprint arXiv:2401.10640}.

\bibitem[{Molnar(2022)}]{christoph2022xai}
Christoph Molnar. 2022.
\newblock Interpretable machine learning: A guide for making black box models explainable.

\bibitem[{Nie et~al.(2020)Nie, Williams, Dinan, Bansal, Weston, and Kiela}]{nie-etal-2020-adversarial}
Yixin Nie, Adina Williams, Emily Dinan, Mohit Bansal, Jason Weston, and Douwe Kiela. 2020.
\newblock Adversarial {NLI}: A new benchmark for natural language understanding.
\newblock In \emph{Proceedings of the 58th Annual Meeting of the Association for Computational Linguistics}. Association for Computational Linguistics.

\bibitem[{Pal et~al.(2022)Pal, Umapathi, and Sankarasubbu}]{pmlr-v174-pal22a}
Ankit Pal, Logesh~Kumar Umapathi, and Malaikannan Sankarasubbu. 2022.
\newblock \href {https://proceedings.mlr.press/v174/pal22a.html} {Medmcqa: A large-scale multi-subject multi-choice dataset for medical domain question answering}.
\newblock In \emph{Proceedings of the Conference on Health, Inference, and Learning}, volume 174 of \emph{Proceedings of Machine Learning Research}, pages 248--260. PMLR.

\bibitem[{Radhakrishnan et~al.(2023)Radhakrishnan, Nguyen, Chen, Chen, Denison, Hernandez, Durmus, Hubinger, Kernion, Luko{\v{s}}i{\=u}t{\.e} et~al.}]{radhakrishnan2023question}
Ansh Radhakrishnan, Karina Nguyen, Anna Chen, Carol Chen, Carson Denison, Danny Hernandez, Esin Durmus, Evan Hubinger, Jackson Kernion, Kamil{\.e} Luko{\v{s}}i{\=u}t{\.e}, et~al. 2023.
\newblock Question decomposition improves the faithfulness of model-generated reasoning.
\newblock \emph{arXiv preprint arXiv:2307.11768}.

\bibitem[{Roemmele et~al.(2011)Roemmele, Bejan, and Gordon}]{roemmele2011choice}
Melissa Roemmele, Cosmin~Adrian Bejan, and Andrew~S Gordon. 2011.
\newblock \href {https://people.ict.usc.edu/~gordon/publications/AAAI-SPRING11A.PDF} {Choice of plausible alternatives: An evaluation of commonsense causal reasoning}.
\newblock In \emph{2011 AAAI Spring Symposium Series}.

\bibitem[{Shinn et~al.(2023)Shinn, Cassano, Gopinath, Narasimhan, and Yao}]{shinn2023reflexion}
Noah Shinn, Federico Cassano, Ashwin Gopinath, Karthik~R Narasimhan, and Shunyu Yao. 2023.
\newblock Reflexion: Language agents with verbal reinforcement learning.
\newblock In \emph{Thirty-seventh Conference on Neural Information Processing Systems}.

\bibitem[{Siegel et~al.(2024)Siegel, Camburu, Heess, and Perez-Ortiz}]{siegel2024probabilities}
Noah~Y Siegel, Oana-Maria Camburu, Nicolas Heess, and Maria Perez-Ortiz. 2024.
\newblock The probabilities also matter: A more faithful metric for faithfulness of free-text explanations in large language models.
\newblock \emph{arXiv preprint arXiv:2404.03189}.

\bibitem[{Talmor et~al.(2019)Talmor, Herzig, Lourie, and Berant}]{talmor-etal-2019-commonsenseqa}
Alon Talmor, Jonathan Herzig, Nicholas Lourie, and Jonathan Berant. 2019.
\newblock \href {https://doi.org/10.18653/v1/N19-1421} {{C}ommonsense{QA}: A question answering challenge targeting commonsense knowledge}.
\newblock In \emph{Proceedings of the 2019 Conference of the North {A}merican Chapter of the Association for Computational Linguistics: Human Language Technologies, Volume 1 (Long and Short Papers)}, pages 4149--4158, Minneapolis, Minnesota. Association for Computational Linguistics.

\bibitem[{Tanneru et~al.(2024)Tanneru, Ley, Agarwal, and Lakkaraju}]{tanneru2024hardness}
Sree~Harsha Tanneru, Dan Ley, Chirag Agarwal, and Himabindu Lakkaraju. 2024.
\newblock On the hardness of faithful chain-of-thought reasoning in large language models.
\newblock \emph{arXiv preprint arXiv:2406.10625}.

\bibitem[{Thomas et~al.(2023)Thomas, Spielman, Craswell, and Mitra}]{thomas2023large}
Paul Thomas, Seth Spielman, Nick Craswell, and Bhaskar Mitra. 2023.
\newblock Large language models can accurately predict searcher preferences.
\newblock \emph{arXiv preprint arXiv:2309.10621}.

\bibitem[{Tian et~al.(2024)Tian, Peng, Song, Jin, Yu, Mi, and Yu}]{tian2024toward}
Ye~Tian, Baolin Peng, Linfeng Song, Lifeng Jin, Dian Yu, Haitao Mi, and Dong Yu. 2024.
\newblock Toward self-improvement of llms via imagination, searching, and criticizing.
\newblock \emph{arXiv preprint arXiv:2404.12253}.

\bibitem[{Touvron et~al.(2023)Touvron, Lavril, Izacard, Martinet, Lachaux, Lacroix, Rozi{\`e}re, Goyal, Hambro, Azhar et~al.}]{touvron2023llama}
Hugo Touvron, Thibaut Lavril, Gautier Izacard, Xavier Martinet, Marie-Anne Lachaux, Timoth{\'e}e Lacroix, Baptiste Rozi{\`e}re, Naman Goyal, Eric Hambro, Faisal Azhar, et~al. 2023.
\newblock Llama: Open and efficient foundation language models.
\newblock \emph{arXiv preprint arXiv:2302.13971}.

\bibitem[{Turpin et~al.(2023)Turpin, Michael, Perez, and Bowman}]{turpin2023language}
Miles Turpin, Julian Michael, Ethan Perez, and Samuel~R Bowman. 2023.
\newblock Language models don't always say what they think: Unfaithful explanations in chain-of-thought prompting.
\newblock \emph{arXiv preprint arXiv:2305.04388}.

\bibitem[{Wang et~al.(2023)Wang, Chuang, Yang, Du, Chang, Zhong, Liu, Xu, Zhou, Cai et~al.}]{wang2023leta}
Guanchu Wang, Yu-Neng Chuang, Fan Yang, Mengnan Du, Chia-Yuan Chang, Shaochen Zhong, Zirui Liu, Zhaozhuo Xu, Kaixiong Zhou, Xuanting Cai, et~al. 2023.
\newblock Leta: Learning transferable attribution for generic vision explainer.
\newblock \emph{arXiv preprint arXiv:2312.15359}.

\bibitem[{Wang et~al.(2022)Wang, Wei, Schuurmans, Le, Chi, Narang, Chowdhery, and Zhou}]{wang2022self}
Xuezhi Wang, Jason Wei, Dale Schuurmans, Quoc Le, Ed~Chi, Sharan Narang, Aakanksha Chowdhery, and Denny Zhou. 2022.
\newblock Self-consistency improves chain of thought reasoning in language models.
\newblock \emph{arXiv preprint arXiv:2203.11171}.

\bibitem[{Wolf et~al.(2019)Wolf, Debut, Sanh, Chaumond, Delangue, Moi, Cistac, Rault, Louf, Funtowicz et~al.}]{wolf2019huggingface}
Thomas Wolf, Lysandre Debut, Victor Sanh, Julien Chaumond, Clement Delangue, Anthony Moi, Pierric Cistac, Tim Rault, R{\'e}mi Louf, Morgan Funtowicz, et~al. 2019.
\newblock Huggingface's transformers: State-of-the-art natural language processing.
\newblock \emph{arXiv preprint arXiv:1910.03771}.

\bibitem[{Yang et~al.(2023)Yang, Wang, Lu, Liu, Le, Zhou, and Chen}]{yang2023large}
Chengrun Yang, Xuezhi Wang, Yifeng Lu, Hanxiao Liu, Quoc~V Le, Denny Zhou, and Xinyun Chen. 2023.
\newblock Large language models as optimizers.
\newblock \emph{arXiv preprint arXiv:2309.03409}.

\bibitem[{Ye and Durrett(2022)}]{ye2022unreliability}
Xi~Ye and Greg Durrett. 2022.
\newblock The unreliability of explanations in few-shot prompting for textual reasoning.
\newblock \emph{Advances in neural information processing systems}, 35:30378--30392.

\bibitem[{Zhao et~al.(2023)Zhao, Chen, Yang, Liu, Deng, Cai, Wang, Yin, and Du}]{zhao2023explainability}
Haiyan Zhao, Hanjie Chen, Fan Yang, Ninghao Liu, Huiqi Deng, Hengyi Cai, Shuaiqiang Wang, Dawei Yin, and Mengnan Du. 2023.
\newblock Explainability for large language models: A survey.
\newblock \emph{arXiv preprint arXiv:2309.01029}.

\end{thebibliography}

\clearpage
\appendix
\onecolumn

\section*{Appendix}

\section{Proof of Theorem}
\label{apdx:proof}

\setcounter{theorem}{0}
\begin{theorem}[\textbf{Latent-Context Intervention Validity for Faithfulness}]
\label{thm:intervention}
Let $f:\mathcal{X}\!\times\!\mathcal{C}\!\to\!\Delta(\mathcal{Y})$ be a language model mapping an input $X$ and latent context $C$ to a predictive distribution over an output space $\mathcal{Y}$. 
Let $E_{NL}$ denote a natural-language explanation of $f(X;C)$, and let $\lnot E_{NL}$ denote its contrary hint. 
Assume that $E_{NL}$ asserts a proposition about a semantic factor $S_E=s(X,C)$, where $s(\cdot)$ extracts the decision-relevant concept, which is latent or retrieved, that the explanation verbalizes. 
Conditioning on $\lnot E_{NL}$ is equivalent to intervening on this factor while holding $(X,C)$ fixed, i.e.,
$f(X;C\mid\lnot E_{NL}) = f(X;C\mid do(S_E\!\leftarrow\!\bar s))$ 
for some contradictory value $\bar s$, and predictions are invariant to any irrelevant text $R$, so $f(X;C)=f(X\cup R;C)$. 
Defining $S_E(X;C)=D(f(X;C),f(X;C\mid\lnot E_{NL}))$, where $D$ is any strictly proper divergence, we have
\begin{align}
    \notag S_E(X;C)=0 &\iff E_{NL}\text{ is non-faithful for }f(X;C),\\
    \notag S_E(X;C)>0 &\iff E_{NL}\text{ is faithful for }f(X;C).
\end{align}
Hence, the contrary-hint score $S_E$ constitutes a valid empirical estimator of causal faithfulness even when the decision-relevant content is not contained in the observed input $X$ but arises from latent or retrieved context.
\end{theorem}

\begin{proof}
By the causal definition of faithfulness, an explanation is faithful if and only if altering the truth value of the decision-relevant semantic factor changes the model’s prediction while holding other causes fixed. Because $E_{NL}$ asserts a proposition about $S_E=s(X,C)$, and conditioning on $\lnot E_{NL}$ implements $do(S_E\!\leftarrow\!\bar s)$ with $(X,C)$ fixed, we have $f(X;C\mid\lnot E_{NL})=f\!\big(X;C \,\big|\, do(S_E\!\leftarrow\!\bar s)\big)$. If $E_{NL}$ is non-faithful, then $S_E$ has no causal influence on $Y$ under $(X,C)$ and the intervention leaves the predictive distribution unchanged, yielding $S_E(X;C)=0$. If $E_{NL}$ is faithful, $S_E$ lies on a causal path to $Y$ under $(X,C)$ and the intervention changes the predictive distribution; strict propriety of $D$ implies $S_E(X;C)>0$. Invariance under irrelevant spans follows by the stated stability condition and applies to both terms inside $D$.
\end{proof}

\begin{corollary}[\textbf{Robustness and Monotonicity}]
\label{cor:robustness}
Under contextual stability, the contrary-hint score remains invariant to irrelevant input variations, satisfying 
\[
S_E(X;C) = S_E(X \cup R;C)
\quad \text{for any semantically irrelevant } R.
\]
Furthermore, if an iterative procedure produces a sequence of explanations 
$\{E_{NL}^{(t)}\}$ such that 
$S_E^{(t+1)}(X;C) \ge S_E^{(t)}(X;C)$ 
for all iterations prior to convergence or a decision flip, then the sequence is non-decreasing in causal faithfulness. 
Consequently, any iteration with $S_E^{(t)}(X;C) > 0$ corresponds to a faithful explanation.
\end{corollary}

\section{Details about Datasets}
\label{appendix:baseline_detail}

The experiments are conducted on the three NLU datasets. The details of the datasets are provided as follows:

\begin{itemize}[leftmargin=*, itemsep=0pt, topsep=-1mm]
    \item \textbf{ECQA~\cite{aggarwaletal2021ecqa}.} ECQA is an extension of the CQA dataset~\cite{talmor-etal-2019-commonsenseqa}. Specifically, based on the CQA dataset, it annotates the positive or negative properties and golden explanations for the QA pairs. Due to API cost budgets, we evaluate our framework on the first 500 instances in the ECQA dataset.
    \item \textbf{TriviaQA LongBench~\cite{2017arXivtriviaqa}.} TriviaQA LongBench (TriviaQA-Long) is a reading comprehension dataset. It includes 300 question-answer-evidence triples sourced from the Longbench~\cite{bai2023longbench} dataset\footnote{https://huggingface.co/datasets/THUDM/LongBench/}. This dataset features question-answer pairs crafted by trivia enthusiasts, accompanied by independently sourced evidence documents, providing supervision for answering these questions.
    \item \textbf{Balanced COPA~\cite{roemmele2011choice, kavumba-etal-2019-choosing}.} The Balanced COPA (COPA) dataset is a collection of 500 questions for commonsense causal reasoning. Each question consists of a premise and two alternatives, where the task is to select the alternative that more plausibly has a causal relation with the premise.
\end{itemize}

\section{Experiment Settings}
\label{apdx:exp_set}
We introduce the experimental settings for evaluating \Algnameabbr{}. Two distinct types of explanation tasks and evaluation settings are as follows.

\paragraph{Fidelity-enhanced Explanation}
In this task, our goal is to produce NL explanations that exhibit a higher fidelity. The fidelity is exploited as a metric to evaluate fidelity. \Algnameabbr{} is evaluated across all testing instances, where an NL explanation is generated for each instance, and the averaged fidelity score is calculated, serving as the reported metric to evaluate fidelity.

\paragraph{Explanation Trigger Prompt Optimization.}
In this task, we aim to optimize the explanation trigger prompt that benefits \Algnameabbr{} in generating better explanations.
The optimization process is conducted on the same dataset, where 30 instances are sampled as a hold-off dataset in each optimization step from the training set. During the optimization process, the fidelity score of a trigger prompt is calculated as the average of the fidelity scores from the selected instances.

\vspace{0.05cm}
\noindent\textbf{Evaluation Metrics.} The quality of the derived NL explanation is evaluated under the fidelity and truthfulness metrics.
The fidelity follows the prior work~\cite{chen2025reasoning}, which observes the discrepancy of the targeted LLMs by incorporating contrary hints to the input.
The evaluation of truthfulness assesses the correlation between the derived NL explanations to the ground-truth explanations
Specifically, we leverage GPT-4o and two well-trained natural language inference (NLI) models, Roberta-Large and XLNet-Large~\cite{nie-etal-2020-adversarial} from the huggingface hub~\cite{wolf2019huggingface}, as the evaluators. With the same evaluators setup, the truthfulness evaluation follows the settings from~\cite{liu2023gpteval}, and uses the evaluation prompt provided in Appendix~\ref{appendix:prompt_eval}. Specifically, the evaluators assess the derived explanations and ground-truth explanations, determining whether the two sentences belong to ``similar content", ``dissimilar content," or ``non-relevant content". Higher the proportion of ``similar content", the more consistent results with ground-truth NL explanations.

\vspace{0.05cm}
\noindent\textbf{Implementation Details.} 
In the experiments, we explore two variants of LLMs as the targeted LLMs $f(\cdot)$: Vicuna-7B~\cite{vicuna2023} and Phi-2~\cite{phi2}, two types of LLMs as the explainers $\textsl{g}(\cdot)$ in \Algnameabbr{}: GPT-3.5-Turbo and Claude-2~\cite{claude}. The LLM agent for generating the contrary hints takes the same LLMs as those used by the explainers. All reported results are calculated from the average scores of 3 times repetitions with the grid search on the performance. The settings for predictors are uniform, with Phi-2 (2.7B) and Viucua-7B receiving identical hyperparameter configurations during the experiments conducted in this study.

\section{Related Work}

\subsection{Post-hoc Explanation}
Post-hoc explanation techniques have undergone significant development and discussion, driven by the widespread adoption of black-box ML models across various data modalities. A multitude of post-hoc algorithms has been introduced from two aspects: local and global explanations~\cite{christoph2022xai, du2019techniques}. Explanations aim to explain the reasoning behind an individual model for each input instance, while global explanations aim to uncover the overall functioning of a complex model~\cite{chuang2023efficient}. Considering various purposes of explanation, the explanation techniques mainly showcase the explanation from two perspectives, including feature attributions and counterfactual examples. Feature attribution aims to provide users with important scores for each feature's impact on model predictions, while counterfactual examples aim to offer alternative instances that explicitly assist users in grasping the model's decision-making process. In recent years, with the growing proficiency and wide usage of black-box LLMs, especially closed-source LLMs service, post-hoc explanations have become increasingly prominent and have garnered significant attention in NLP research due to the inaccessibility of LLMs' model weights and structure~\cite{zhao2023explainability}. 

\subsection{Explainability of LLMs}
The majority of explanation efforts in LLM research have centered on delivering explanations. One group of studies calculates importance scores for specific tokens~\cite{lopardo2023faithful, huang2023can}, another line of progress generates NL explanations by leveraging the pre-trained LLMs with internal model knowledge sources~\cite{kumar2020nile,chen2023lmexplainer, menon2023mantle}, the other group of work leverages LLMs themselves to generate chain-of-thought (CoT) reasoning~\cite{lanham2023measuring, radhakrishnan2023question, chen2023models, chen2023models} as the self-explanations through the one feed-forward inference process. Furthermore, some studies aim to yield counterfactual explanations by pre-trained LLMs to assist users in better understanding the decision-making process from LLMs~\cite{chen2021kace, chen2023models}. Although NL explanations offer fantastic human-understandable insights than token-wise explanations, the explanations can lose their fidelity via one feed-forward inference process of pre-trained LLMs. Unreliability and non-fidelity of NL explanations are still a concern~\cite{ye2022unreliability, turpin2023language}. Given our primary aim of producing faithful explanations, our efforts are to generate NL explanations to improve the likelihood of accurately representing the decision-making process of LLMs.

\subsection{LLMs as Optimizers}
LLMs as optimizers is a novel paradigm, describing optimization problems in natural language and utilizing the reasoning capabilities of LLMs for optimizing~\cite{yang2023large}.
Depicting optimization problems in natural language enables the optimization of diverse tasks without defining formal specifications, such as prompt optimization~\cite{yang2023large, cheng2023black, guo2023connecting}, agent learning~\cite{shinn2023reflexion}, and model labeling~\cite{thomas2023large}.
Based on this optimization paradigm, our work introduces a generative explanation framework with a novel estimation method of sentence-level fidelity.

\newpage
\section{Hyper-parameter Settings of \Algnameabbr{}}
\label{appendix:hyp}

The hyper-parameters of \Algnameabbr{} are given in Table~\ref{tab:hyperparam}. The configuration for explainers is consistent across Claude-2 and GPT-3.5-Turbo, provided that the parameters are adjustable. Likewise, the settings for predictors are uniform, with Phi-2 and Viucua-7B receiving identical hyperparameter configurations during the experiments conducted in this study.

\begin{table}[h!]
\vspace{2mm}
\centering
\begin{tabular}{l|l|c|c|c}
\toprule
& Dataset & ECQA & TriviaQA-Long & COPA \\
\midrule
\multirow{3}{*}{\makecell[l]{Fidelity-enhanced \\Optimization}} 
& Optimization Steps & 20 & 20 & 20 \\
& Temperature of Predictor LLMs & 0.7 & 0.5 & 0.7 \\
& Temperature of Explainer LLMs & 0.9 & 0.9 & 0.9 \\
& Top-P of Explainer LLMs & 0.9 & 0.9 & 0.9 \\
\midrule
\multirow{3}{*}{\makecell[l]{Trigger-oriented \\Optimization}}
& Optimization Steps & 50 & 100 & 100 \\
& Sampled Instances & 30 & 30 & 30 \\
& Temperature of Predictor LLMs & 0.7 & 0.5 & 0.7 \\
& Temperature of Explainer LLMs & 0.9 & 0.9 & 0.9 \\
& Top-P of Explainer LLMs & 0.9 & 0.9 & 0.9 \\
\bottomrule
\end{tabular}
\vspace{3mm}
\caption{Hyper-parameters and optimization settings in \Algnameabbr{}.}
% \vspace{-7mm}
\label{tab:hyperparam}
\end{table}

\section{Computation Infrastructure and Costs}
\label{apx:infra}

\subsection{Computation Infrastructure}
For a fair comparison of testing algorithmic throughput, the experiments are conducted based on the following physical computing infrastructure in Table~\ref{tab:computing_infrastructure}. 

\begin{table}[h!]
\centering
\begin{tabular}{l|c}
\toprule
Device Attribute & Value \\
\midrule
Computing infrastructure & GPU \\
GPU model & Nvidia-A40 \\
GPU number & 1 \\
GPU Memory & 46068 MB \\
\bottomrule
\end{tabular}
\vspace{1mm}
\caption{Computing infrastructure for the experiments.}
% \vspace{-7mm}
\label{tab:computing_infrastructure}
\end{table}

\subsection{Computation Costs}
The computational costs associated with \Algnameabbr{} primarily differ from the inference costs of local LLMs and the expenses related to API-accessed LLMs. The computational costs depend on the parameter scale and variants of LLMs used in the \Algnameabbr{} framework, shown in Table~\ref{tab:cost_gpt} and \ref{tab:cost_vicuna}.

\begin{table}[h!]
\centering
\begin{tabular}{l|c|c|c}
\toprule
 & ECQA & TrivaQA & COPA \\
\midrule
Execution Time (Sec.) & $\sim$3 & $\sim$5 & $\sim$3 \\
Execution Cost (\$) & $\sim$0.01 & $\sim$0.04 & $\sim$0.01 \\
\bottomrule
\end{tabular}
\vspace{1mm}
\caption{Computing costs of \Algnameabbr{} with GPT-3.5 on each dataset.}
\vspace{-4mm}
\label{tab:cost_gpt}
\end{table}

\begin{table}[h!]
\centering
\begin{tabular}{l|c|c|c}
\toprule
 & bs=32 & bs=64 & bs=96 \\
\midrule
Execution Time (Sec.) & $\sim$3 & $\sim$5 & $\sim$3 \\
Memory Cost (GB) & $\sim$28GB & $\sim$43GB & $\sim$59GB \\
\bottomrule
\end{tabular}
\vspace{1mm}
\caption{Computing costs of \Algnameabbr{} with Vicuna-7B under different batch size (bs).}
\label{tab:cost_vicuna}
\end{table}

\newpage
\section{Additional Experimental results of \Algnameabbr{}}
\label{appendix:trg}

\subsection{Optimization Procedure of derived explanations}
We demonstrate more evaluation results on derived explanations from \Algnameabbr{}. The outcomes depicted in Figure~\ref{appendix:fig:local_exp_accum} reveal that \Algnameabbr{} attains notably higher fidelity scores across all three datasets following 20 steps of optimization. Additionally, Figure~\ref{appendix:fig:local_exp_accum} illustrates the evolution of the optimization process during the generation of explanations.

\begin{figure*}[h]
\centering
\subfigure{
\includegraphics[width=0.3\linewidth]{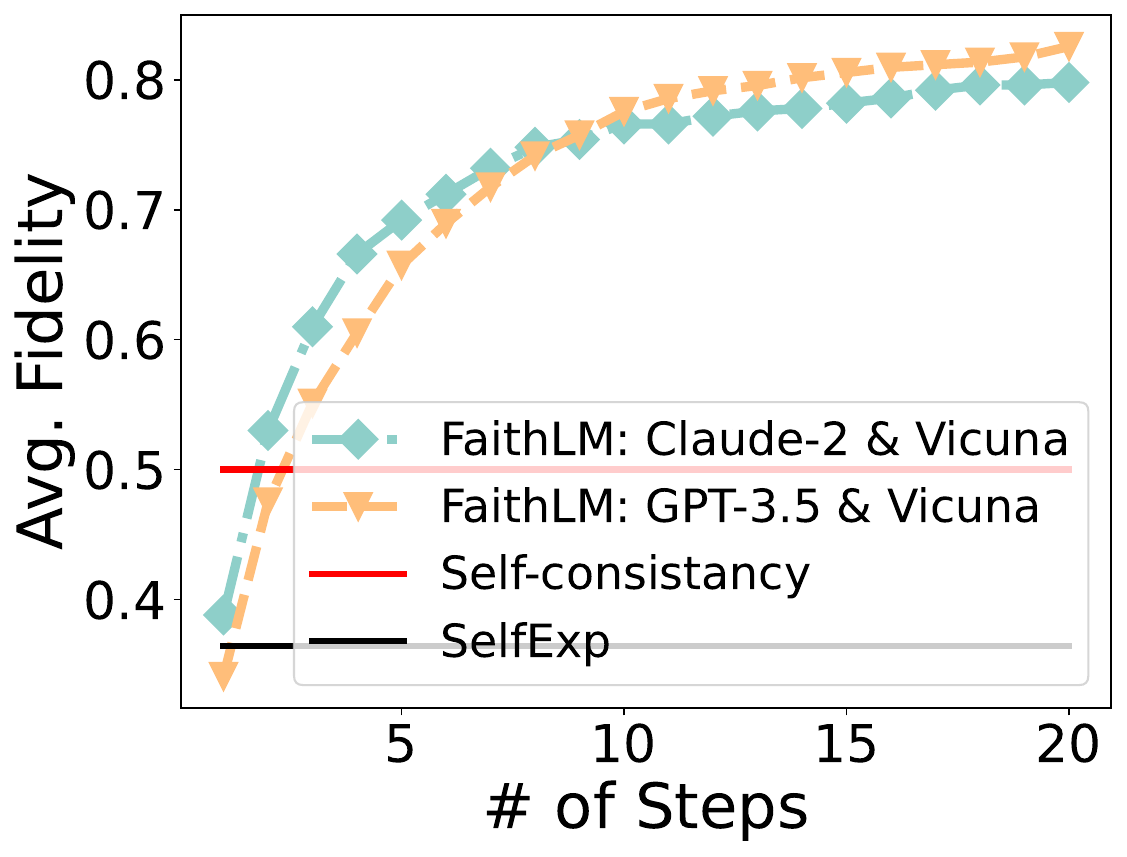}
}~~
\subfigure{
\includegraphics[width=0.3\linewidth]{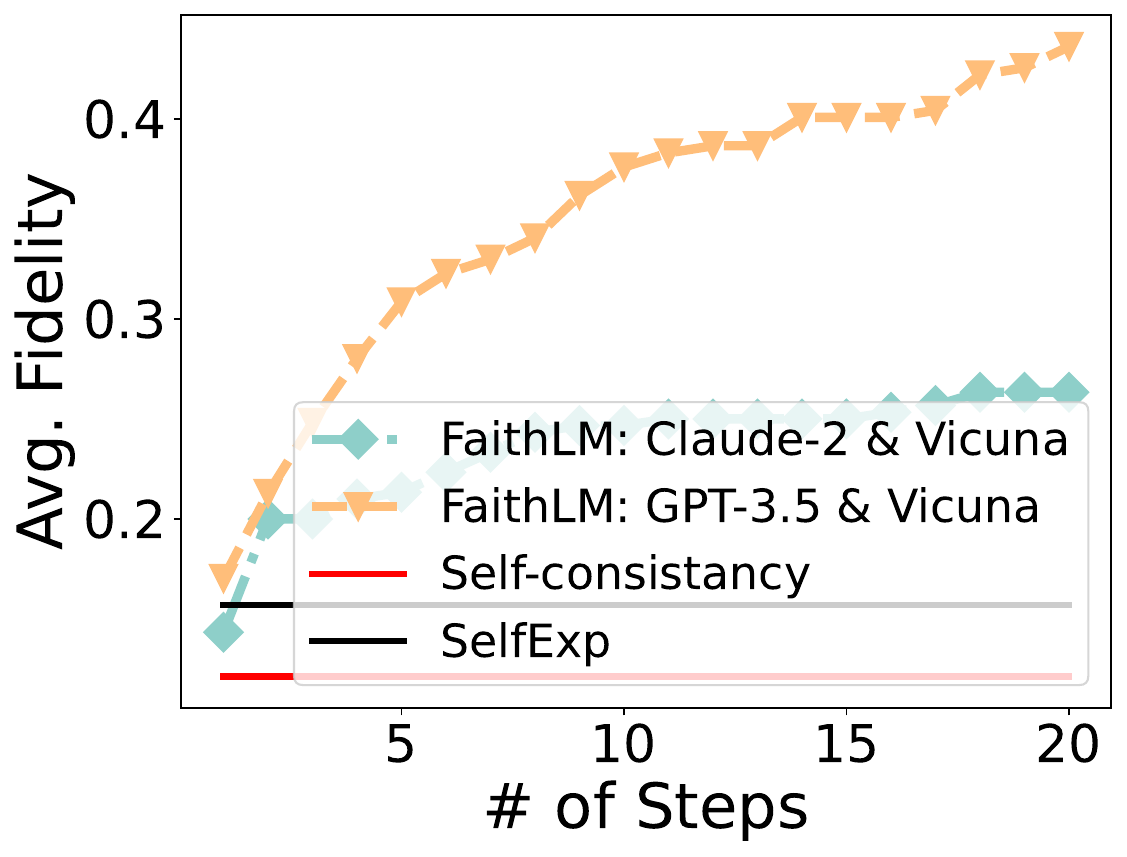}
}~~
\subfigure{
\includegraphics[width=0.3\linewidth]{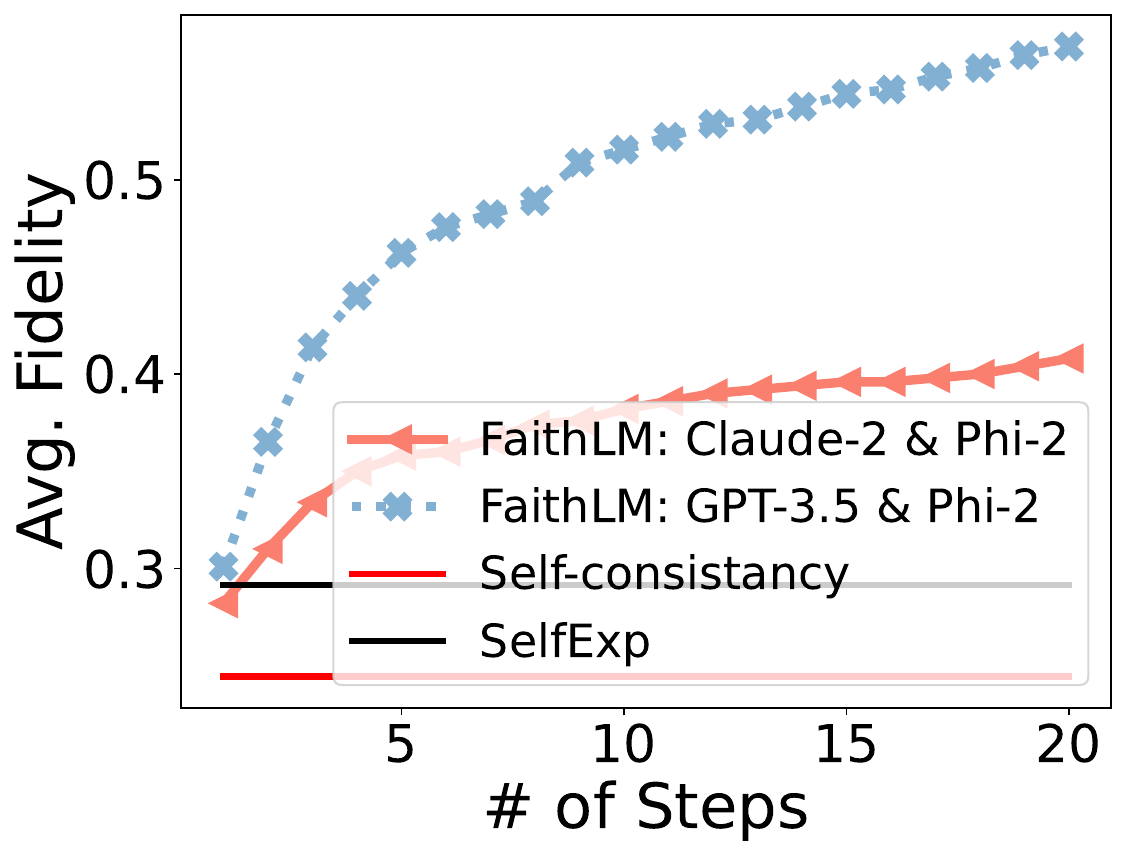}
}
\vspace{-0.3cm}
\caption{The fidelity evaluation of derived explanations from \Algnameabbr{} under different settings of predictors and explainers.}
% \vspace{-0.2cm}
\label{appendix:fig:local_exp_accum}
\end{figure*}

\subsection{Additional Optimization Curve of Explanation Trigger Prompt}
We demonstrate more evaluation results on the optimization curve of explanation trigger prompts of \Algnameabbr{}. The optimization curve shown in Figure~\ref{appendix:fig:opt_curve} generally displays an upward trend with the progression of steps, interspersed with several fluctuations throughout the optimization process. This suggests that \Algnameabbr{} can successfully generate improved explanation trigger prompts after optimization.

\begin{figure*}[h]
\centering
\subfigure{
\includegraphics[width=0.31\linewidth]{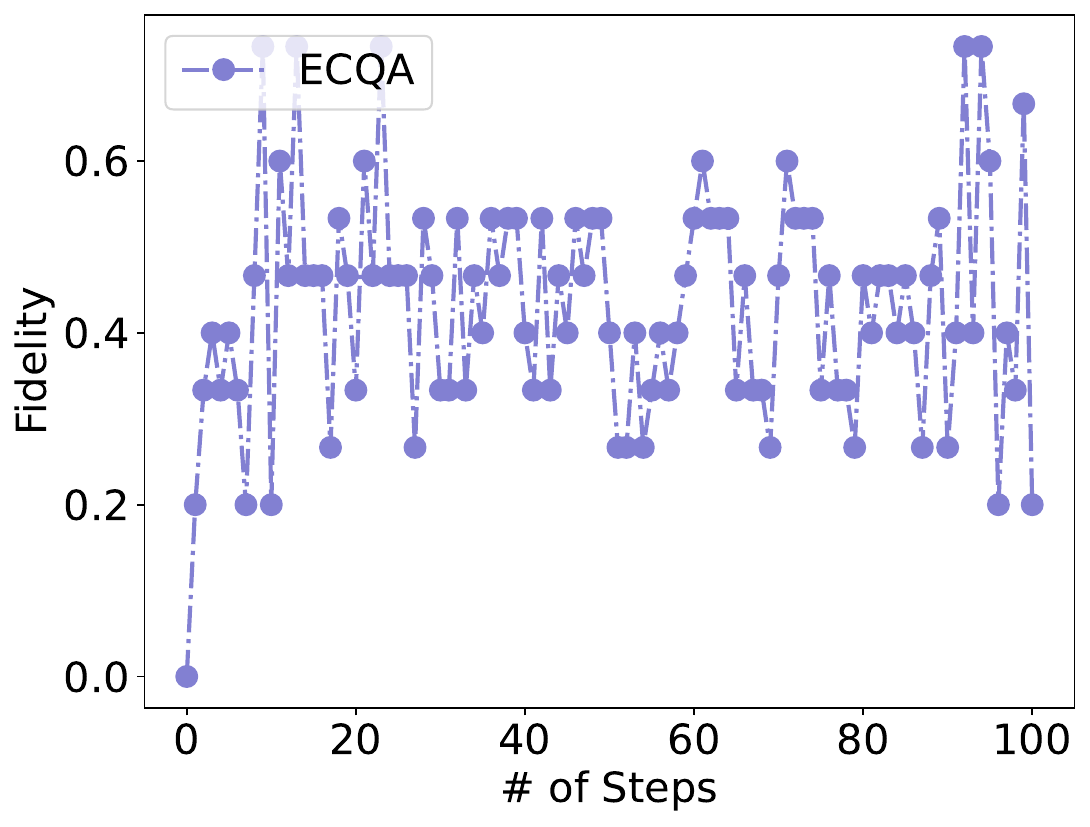}
}~~
\subfigure{
\includegraphics[width=0.32\linewidth]{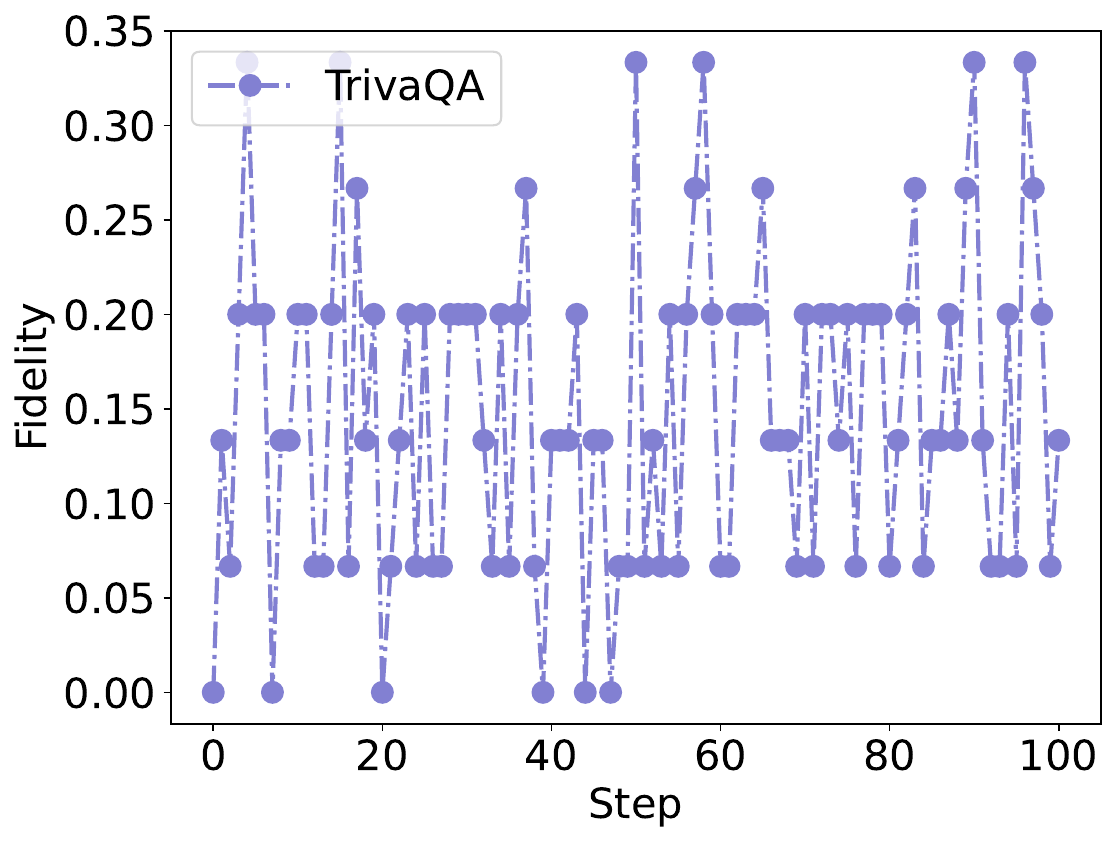}
}~~
\subfigure{
\includegraphics[width=0.31\linewidth]{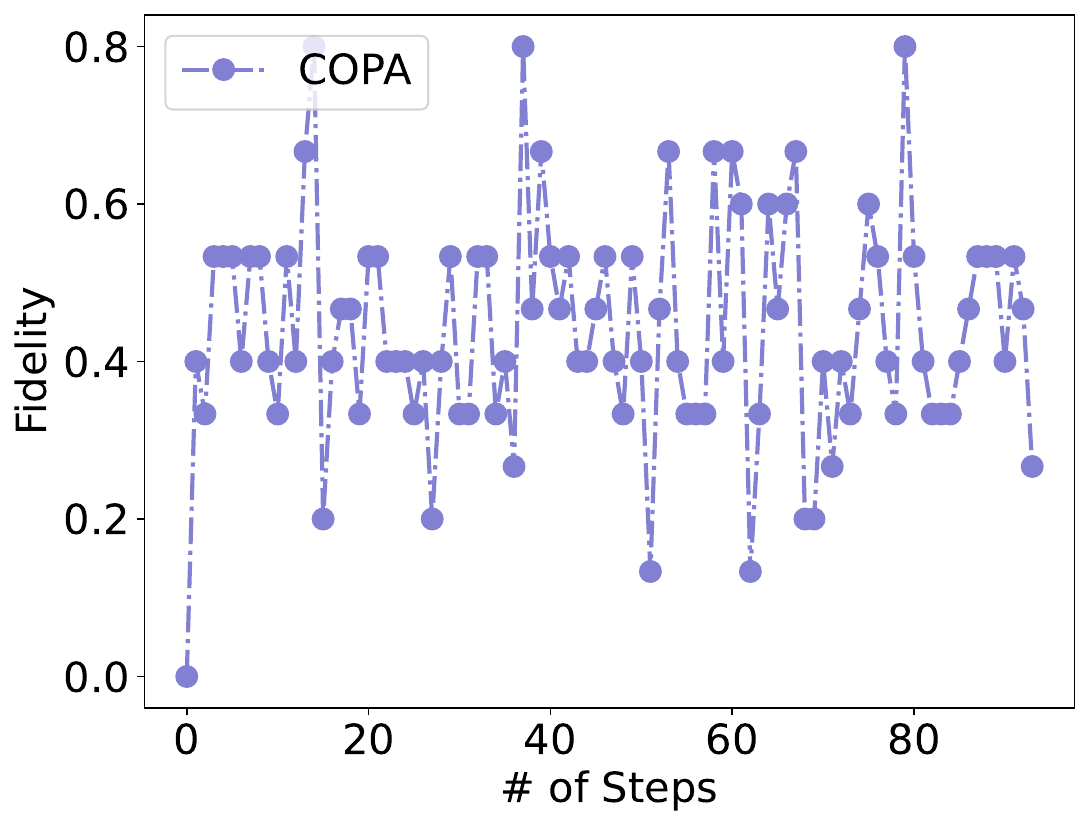}
}
\vspace{-0.3cm}
\caption{The optimization curve of explanation trigger prompts on ECQA (left), TriviaQA-Long (middle), and COPA dataset (right). }
\label{appendix:fig:opt_curve}
\end{figure*}

\subsection{Additional Experiments on Diverse Domains of Dataset}
We further have conducted additional experiments on one new MedMCQA dataset~\cite{pmlr-v174-pal22a} in the healthcare domain. We evaluate the \Algnameabbr{} using a fidelity assessment under Natural Language Explanation Generation settings. All experimental configurations follow the settings in Section~\ref{sec:exp}. The experimental results are shown in the table below. We observe that \Algnameabbr{} outperforms the baseline method, which is consistent with the experimental results across other domain datasets that were evaluated in our work.

\begin{table}[h!]
\centering
\begin{tabular}{l|c|c|c}
\toprule
 & \texttt{SelfExp} & \texttt{Self-consistency} & \Algnameabbr{} \\
\midrule
Fidelity & 0.6956 & 0.4715 & \textbf{0.9565} \\
\bottomrule
\end{tabular}
\vspace{3mm}
\caption{Additional experimental results on MedMCQA dataset.}
\vspace{-3mm}
\end{table}

\clearpage
\subsection{Additional Experiments on Truthfulness Evaluation}
We evaluate all baseline methods and \Algnameabbr{} under multiple settings to assess how closely the derived explanations match the ground-truth rationales. A GPT-based evaluator assigns a GPT-Score from 1 to 5, with higher values indicating greater semantic similarity. Explanations tagged as “similar content” or scoring near 5 are treated as matches.
We report the truthfulness of generated explanations in Figure~\ref{fig:local_exp_gt_appendix}. A GPT-based evaluator assigns a truthfulness score from 1 to 5 by checking factual consistency with the input, the task label, and commonsense knowledge. Explanations tagged as factually consistent or scoring near 5 are counted as truthful. Across all settings, \Algnameabbr{} attains the highest mean truthfulness score and the largest fraction of truthful explanations, indicating that our optimization improves not only fidelity but also factual quality of the produced rationales.

\begin{figure}[h]
    \centering
    \includegraphics[width=0.35\linewidth]{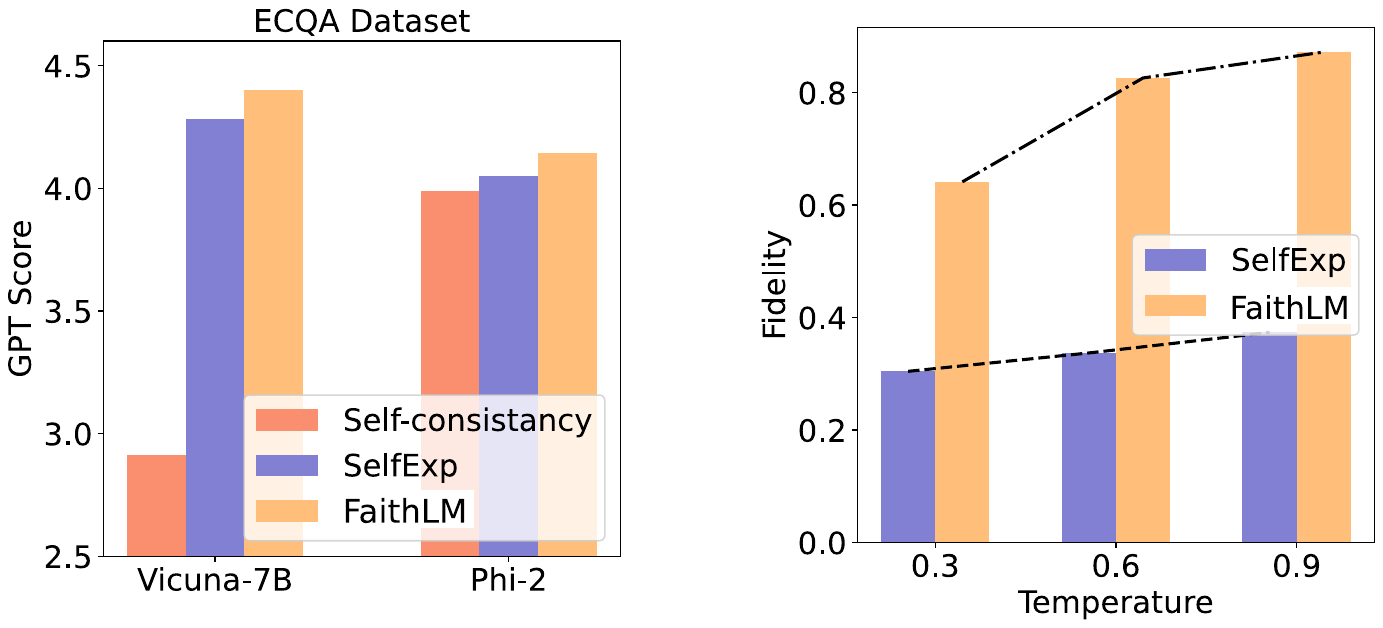}
    \caption{Truthfulness evaluation with ground-truth explanation under GPT-score settings.}
\vspace{-0.35cm}
\label{fig:local_exp_gt_appendix}
\end{figure}

\subsection{Additional Experiments of Contrary Hints.}
We here showcase the results of using NLI classifiers to evaluate the quality of contrary hints.

\begin{figure}[h]
    \centering
    % \subfigure[Evaluated by NLI classifiers]
    % {
    %     \includegraphics[width=0.38\linewidth]{figures/noncondi_copa.pdf}
    %     \includegraphics[width=0.38\linewidth]{figures/noncondi_ecqa.pdf}
    % \label{fig:cont_class}
    % }
    \includegraphics[width=0.38\linewidth]{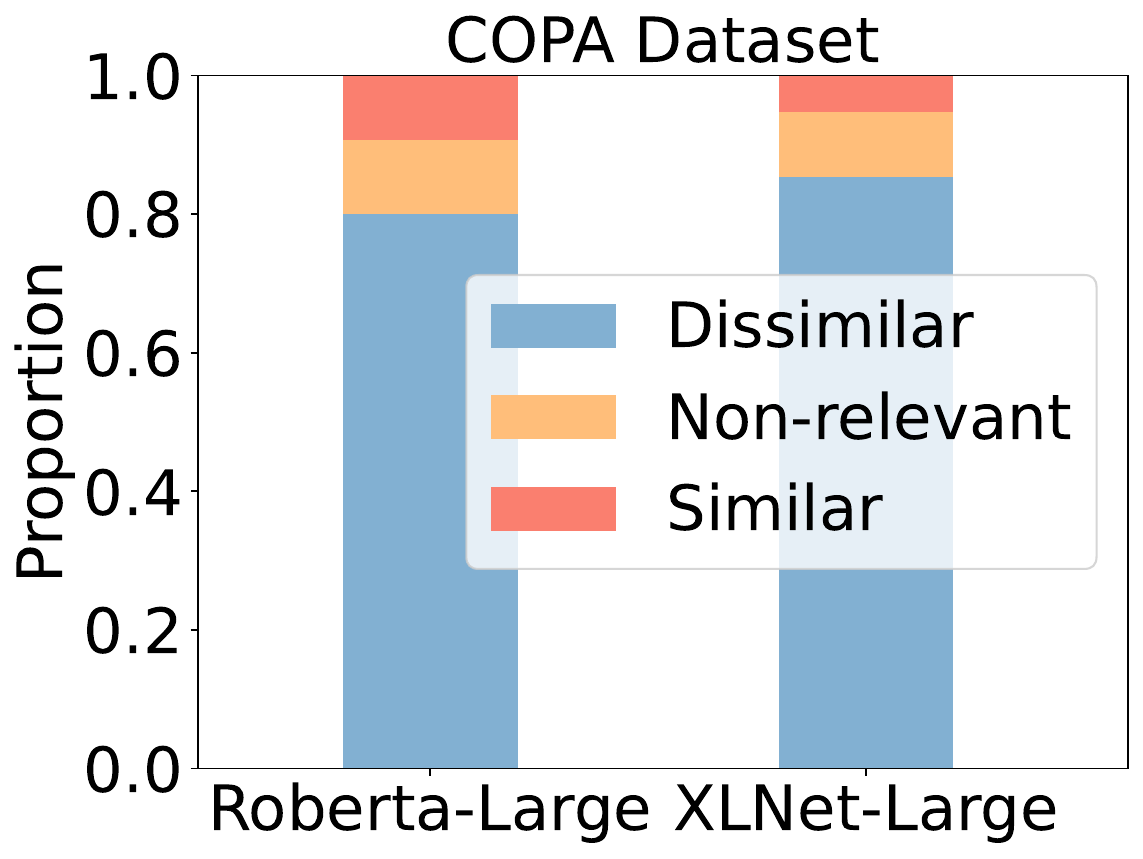}
    \includegraphics[width=0.38\linewidth]{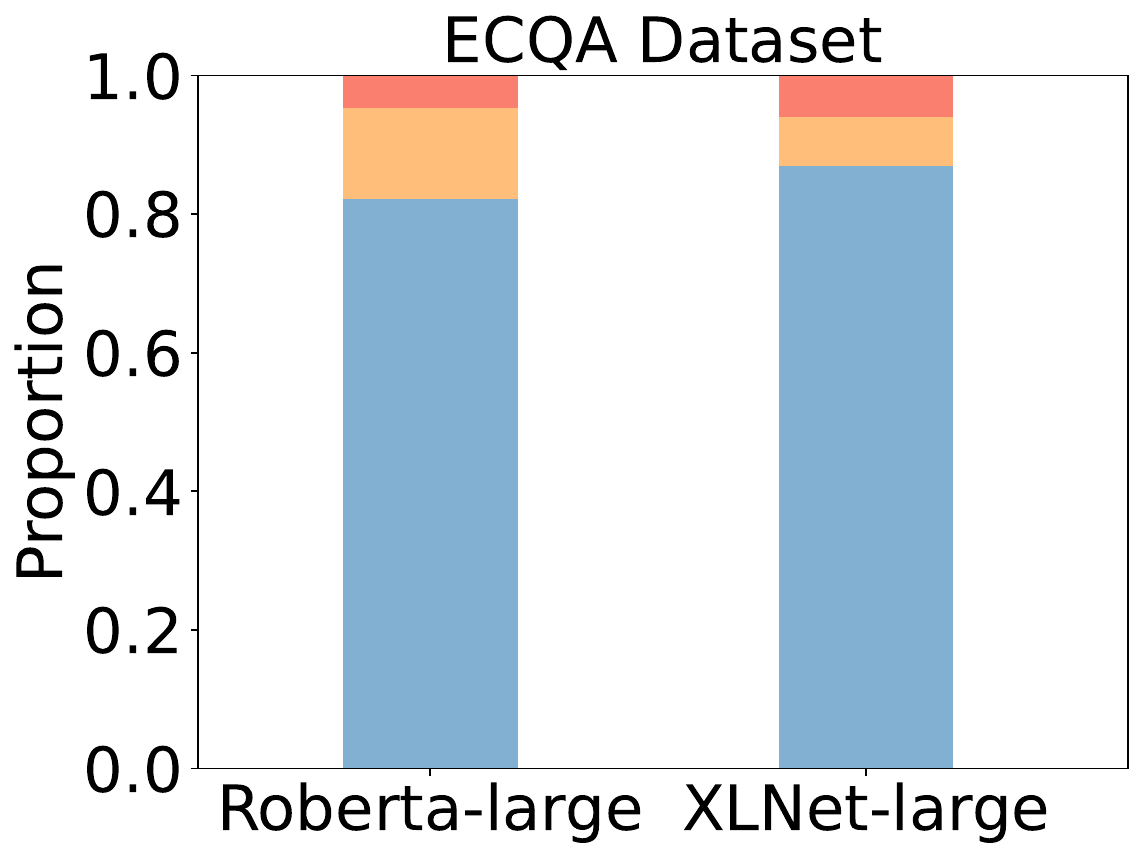}
    \caption{Quality Evaluation results of Contrary Hints via NLI classifiers}
\label{fig:cont_class}
\end{figure}

\clearpage
\section{Details of Evaluation Prompt Usage}
\label{appendix:prompt_eval}
We provide a listing of the evaluation prompts in Table~\ref{tab:eval} utilized in assessing the performance of \Algnameabbr{}. The first row reveals the evaluation prompt on comparing the derived explanation with the ground-truth (GT) explanation in the ECQA dataset in Section~\ref{sec:local_eff}; and the second row demonstrates the evaluation prompt on activating the GPT classifier and the GPT scorer for assessing contrary hints in Section~\ref{sec:abl_non}.

\vspace{0.3cm}
\begin{table*}[h!]
    \small
    \centering
    \begin{tabular}{p{4cm}|p{9cm}} \toprule
        Evaluation Task & Evaluation Prompts \\ 
        \toprule\toprule
        Ground-truth Explanation & 
        Given a user instruction and two AI assistant responses, your job is to classify whether the relation of two responses in S1 and S2 belongs to G-1, G-2, or G-3. The meaning of class is as follows: (G-1) relevant contents, (G-2) irrelevant contents, or (G-3) irrelevant contents. Judge responses holistically, paying special attention to whether two responses have similar contents. Judge responses with only ONE class label as your final answer. \textbf{S1:\{\textit{derived explanation}\}. S2:\{\textit{GT-Explanation}\}.} Please ONLY response your in either G-1, G-2, or G-3; THERE SHOULD BE NO OTHER CONTENT INCLUDED IN YOUR RESPONSE.\\\toprule
        GPT classifier for contrary hints &  Given a user instruction and two AI assistant responses, your job is to classify whether two responses in S1 and S2 belong to G-1, G-2, or G-3. The meaning of class is as follows: (G-1) same semantic meaning, (G-2) opposite semantic meaning, and (G-3) no relation. Judge responses holistically, paying special attention to whether two responses have the same semantic meaning. Judge responses with only ONE class label as your final answer. \textbf{S1:\{\textit{derived explanation}\}. S2:\{\textit{contrary hints}\}.} Please ONLY respond in either G-1, G-2, or G-3; THERE SHOULD BE NO OTHER CONTENT INCLUDED IN YOUR RESPONSE. \\\toprule
        GPT scorer of contrary hints & Given a user instruction and two AI assistant responses, your job is to rate from ONE to FIVE to judge whether two responses in S1 and S2 have the same semantic meaning or not. A FIVE score refers to being totally the same, and ONE score refers to being totally the opposite. Judge responses holistically, paying special attention to whether two responses have the same semantic meaning. The judge responds with the rates between ONE and FIVE. \textbf{S1:\{\textit{derived explanation}\}. S2:\{\textit{contrary hints}\}.} Please ONLY respond to the rate value; THERE SHOULD BE NO OTHER CONTENT INCLUDED IN YOUR RESPONSE. \\
    \bottomrule
    \end{tabular}
    \caption{Evaluation Prompts given to GPT-3.5-Turbo used in assessing the efficacy of \Algnameabbr{}.}
    \label{tab:eval}
\end{table*}

\clearpage
\subsection{Robustness Analytics of Configuration (RQ3)}
In this section, the robustness test of the explainer LLMs is conducted under the analytics of hyper-parameters that are highly dependent on the outputs of LLMs. We focus on two different hyper-parameters: Temperature and Top-p. The experiments are conducted under the explainer GPT-3.5-Turbo and the predictor Vicuna-7B. We evaluate the following temperatures and top-p of the explainer LLM in the range of \{0.3, 0.6, 0.9\}.
The results are shown in Figure~\ref{fig:rob}.
We observe that the explainer LLMs perform inferior when the temperature and Top-p are low, reflecting that the lower exploration of explainer LLM may degrade the optimization ability in explanation generation. The explainer LLMs are encouraged to obtain the temperatures and top-p around 0.9. The small values of the temperatures and top-p may lead to low flexibility in updating new explanations. In contrast, large temperatures and top-p may impact explainer LLMs disobeying the given optimization trajectory. Thus, in the main experiments, all reported performances are under the settings of temperature 0.9 and top-p 0.9, achieving the best performance for generating explanations.

\begin{figure}[h]
\centering
\subfigure{
    \includegraphics[width=0.35\linewidth]{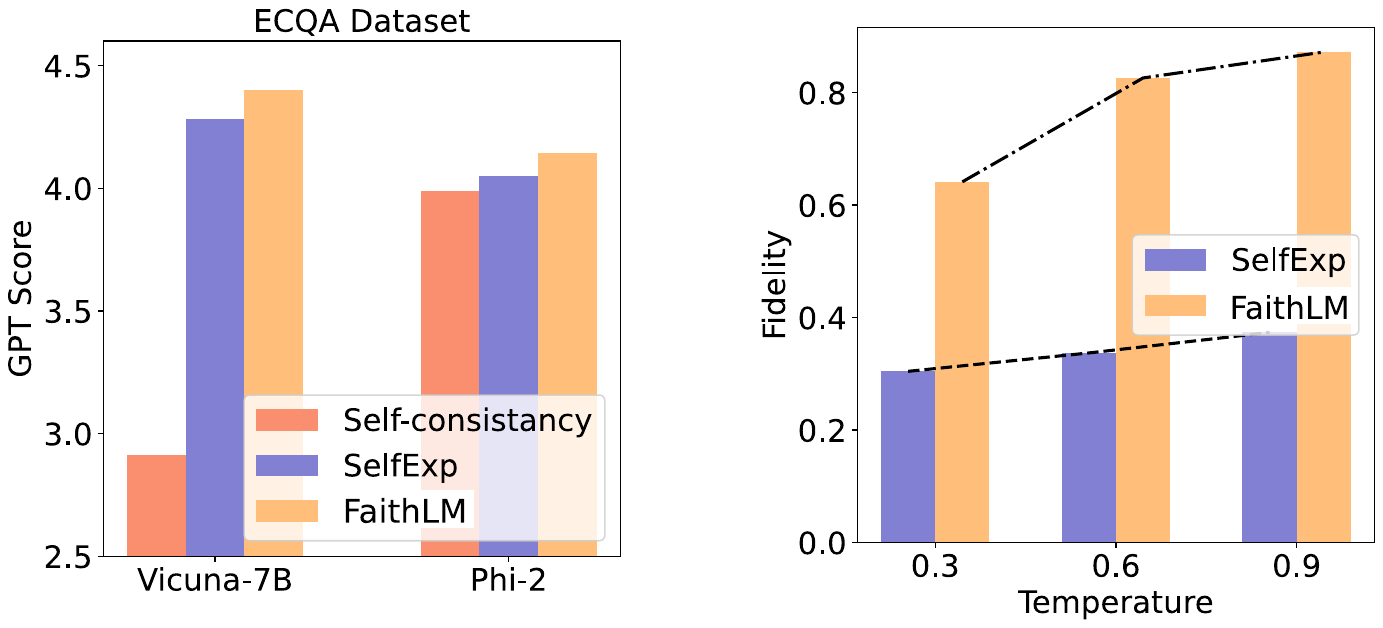}
}
\subfigure{
    \includegraphics[width=0.35\linewidth]{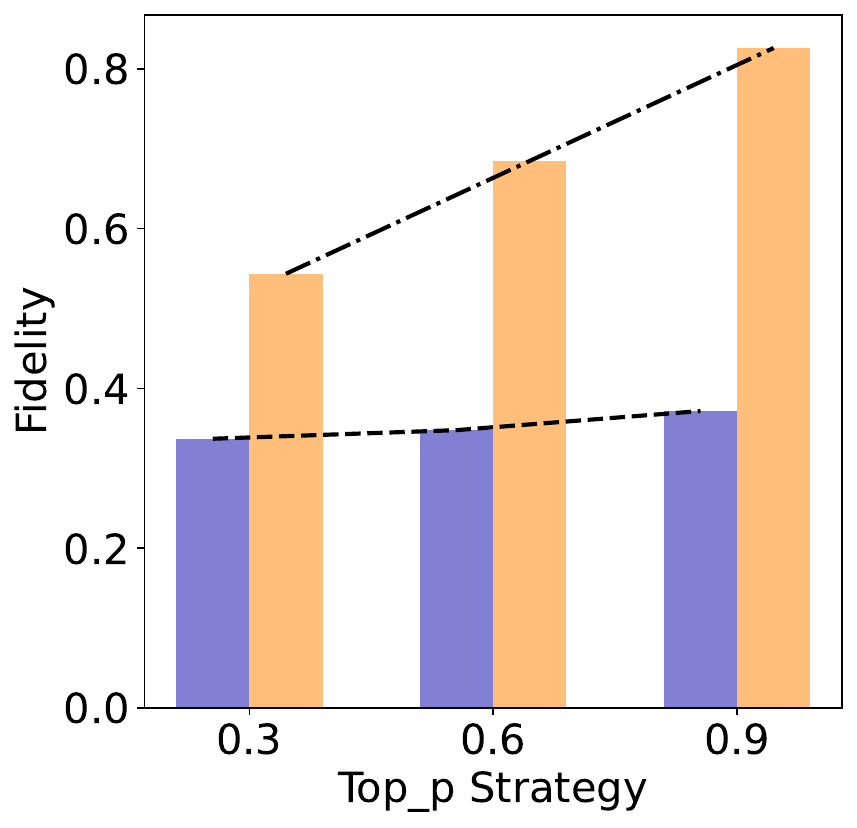}
}
\vspace{-0.2cm}
\caption{Robustness Analytics of \Algnameabbr{}: Temperature (left) and Top-p strategies (right).}
\label{fig:rob}
\end{figure}

\section{Details of Prompts Usage in \Algnameabbr{}}
\label{appendix:prompt_meta}
We provide a listing of the prompts in Table~\ref{tab:eval} utilized in \Algnameabbr{} in different tasks. The first row demonstrates the initial explanation trigger prompt leveraging in both fidelity-enhanced optimization and trigger-oriented optimization. The second row shows the prompt for the LLM agent to generate the contrary hints.

\begin{table*}[h!]
    \small
    \centering
    \begin{tabular}{p{3cm}|p{9cm}} \toprule
        Conducted Task & Evaluation Prompts \\ 
        \toprule\toprule
        Explanation Generation & Please provide objective explanations of why the model generates the answers to the given questions based on your thoughts. Explain the reason why the model provides the answer, no matter if it is wrong or correct. Make sure not to answer the questions or provide any suggestions to better answer the questions by yourself. \textbf{Q:\{\textit{Question}\}. A:\{\textit{Targeted LLM-generated Answer}\}.} \\\toprule
        contrary hints Generation & Please generate one example of obtaining the opposite meaning from a given sentence. Make sure you output sentences only. \textbf{Sentences:\{\textit{derived explanation}\}.} \\\toprule
    \bottomrule
    \end{tabular}
    \caption{The example of prompts that are given to two explainer LLMs and LLM agent for contrary hint.}
    \label{tab:meta}
\end{table*}

\clearpage
\section{Trajectory System Prompts Usage in \Algnameabbr{}}
\label{appendix:prompt_xllm}

We present a detailed listing of the trigger-oriented trajectory prompt in Figure~\ref{appendix:fig:trg} and the explanation-oriented prompt in Figure~\ref{appendix:fig:exp}, as utilized within the \Algnameabbr{} framework.

\subsection{Trigger-oriented Trajectory Prompt}

\begin{figure*}[h!]
\centering
\fcolorbox{black}{white!10}{\parbox{.85\linewidth}
    {
        \textbf{System instruction:} Your task is to generate the general prompts $\textless$INS$\textgreater$ for language model generating model explanations of each question. Below are some previous prompts with their scores in the Inputs. The score is calculated as the flipping answer rates and ranges from 0 to 1. \\\\

        \textbf{Inputs:} The following exemplars show how to apply your text:\\
        Text: Please provide objective explanations of why model generates the answers. \\
        Score: 0.21 \\

        Text: Provide a concise, objective explanation of only the key reasoning or assumptions that likely led the model to generate this specific response. \\
        Score: 0.53 \\

        $\cdots\cdots$ \\

        \textbf{Trajectory Instruction:} Generate a prompt $\textless$INS$\textgreater$ that is different from all prompt $\textless$INS$\textgreater$ in Inputs above and has a higher score than all the prompts $\textless$INS$\textgreater$ from Inputs. The prompts should begin with $\textless$INS$\textgreater$ and end with $\textless$\/INS$\textgreater$ and follow the format of the examples in Inputs. The prompts should be concise, effective, and generally applicable to all problems above. \\\\

        \textbf{Response:} $\textless$A Newly Generated Trigger Prompt$\textgreater$
    }
}
\caption{A examples of \textbf{trigger-oriented trajectory prompt}. This prompt populates in both LLM explainers, which are Cluade2 and GPT-3.5-Turbo. The output of \Algnameabbr{} optimized under trigger-oriented trajectory prompt is append after the \textbf{Response} label.}
% \vspace{-0.55cm}
\label{appendix:fig:trg}
\end{figure*}

\newpage
\subsection{NL Explanation-oriented Trajectory Prompt}
\begin{figure*}[h!]
\centering
\fcolorbox{black}{white!10}{\parbox{.85\linewidth}
    {
        \textbf{System instruction:} You have some texts along with their corresponding scores. The texts are the possible explanation of the following given question and answer. The texts are arranged in random order based on their scores, where higher scores indicate better quality. The scores are calculated as how relative the texts are toward the given question and answer as the explanation. The scores range from 0 to 1 based on your output text. \\\\

        \textbf{Inputs:} The following exemplars show how to apply your text:\\
        Text:  The model generates the answer "farmland" because an apple tree is likely found in abundance in farmland. \\
        Score: 0.0 \\

        Text: The model generates the answer "farmland" because \textbf{apple trees require open spaces and fertile soil}, both of which are commonly found in farmland.' \\
        Score: 1.0 \\ 

        $\cdots\cdots$ \\

        \textbf{Trajectory Instruction:} You replace $\textless$EXP$\textgreater$ with your text. We say your output is bad if your output obtains lower scores than the previous text, and we say your output is good if your output obtains higher scores than the previous text.
        Please provide new objective text to describe why the answers are given to the questions based on your thoughts. Explain the reason, no matter if the answer is wrong or correct. Make sure not to answer the questions or provide any suggestions to better answer the questions by yourself. Every explanation should begin with $\textless$EXP$\textgreater$. Make sure not to repeat the input queries and answers. Please only output the explanation sentences.
        \\\\

        \textbf{Response:} $\textless$A Newly Generated Trigger Prompt$\textgreater$
    }
}
\caption{A examples of \textbf{explanation-oriented trajectory prompt}. This prompt populates in both LLM explainers, which are Cluade2 and GPT-3.5-Turbo. The output of \Algnameabbr{} optimized under trigger-oriented trajectory prompt is append after the \textbf{Response} label.}
% \vspace{-0.55cm}
\label{appendix:fig:exp}
\end{figure*}

\clearpage
\section{Case Studies of \Algnameabbr{}}
\label{appendix:case}

\subsection{Explanations Generation}
We demonstrate the case studies of explanation results generated by \Algnameabbr{} after fidelity-enhanced optimization in the following two Figure~\ref{appendix:fig:local_exp_copa} and Figure~\ref{appendix:fig:local_exp_ecqa}. The results show that \Algnameabbr{} can generate faithful explanations that truly express the decision-making process of targeted LLMs.

\vspace{0.2cm}
\begin{figure*}[h!]
\centering
\fcolorbox{black}{gray!10}{\parbox{.9\linewidth}
    {
        \hspace{-1cm}\textbf{Case Study 1} \\
        \textbf{Question:} What is the cause of the Promise? Premise: The movie tickets sold out. \\
        Choices: [choice] It was opening day for the movie. [choice] The movie received poor reviews. \\
        \underline{Ground Truth:} The movie received poor reviews. \\
        \underline{\textbf{LLM Answer:}} It was opening day for the movie. \textbf{(wrong)} \\
        \textbf{derived explanation:} The model likely selected that answer because opening day tends to have very high demand for movies, often resulting in sold out showings. While there could be other reasons for the sell out, high opening day demand is a reasonable explanation given the limited context provided. 
        \hangindent=1cm \hangafter=0
        \\
        
        \hspace{-1cm}\textbf{Case Study 2} \\
        \textbf{Question:} What is the effect of the Promise? Premise: I drilled a hole in the wall. \\
        Choices: [choice] A mouse crawled out of the hole. [choice] Dust blew out of the hole. \\
        \underline{Ground Truth:} Dust blew out of the hole. \\
        \underline{\textbf{LLM Answer:}} Dust blew out of the hole. \textbf{(correct)} \\
        \textbf{derived explanation:} Dust exiting the hole in the wall is a reasonable effect to expect when drilling into a wall. Without further context, there is no reason to assume a mouse coming out of the hole is likely based solely on the premise that a hole was drilled. The model selected the choice directly supported by the information provided in the premise.
        \hangindent=1cm \hangafter=0
    }
}
\caption{Case Study: \textbf{derived explanations} (Dataset: \textbf{COPA}; Targeted Model $f(\cdot)$: \textbf{Claude-2}; Explainer $\textsl{g}(\cdot)$: \textbf{Vicuna})}
% \vspace{-0.55cm}
\label{appendix:fig:local_exp_copa}
\end{figure*}

\vspace{0.2cm}
\begin{figure*}[h!]
\centering
\fcolorbox{black}{gray!10}{\parbox{.9\linewidth}
    {
        \hspace{-1cm}\textbf{Case Study 1} \\
        \textbf{Question:} Where would you get a dog if you do not have one? \\
        Choices: [choice] pet store [choice] outside [choice] neighbor's house [choice] park [choice] rug \\
        \underline{\textbf{Ground Truth}:} store \\
        \underline{\textbf{LLM Answer:}} store \textbf{(correct)} \\
        \textbf{Derived explanation:} The model suggested "pet store" as the place to get a dog if you don't already have one because pet stores are a common retail business where people can purchase dogs as pets. 
        \hangindent=1cm \hangafter=0
        \\
        
        \hspace{-1cm}\textbf{Case Study 2} \\
        \textbf{Question:} Where would you put coins if you want to bring them with you? \\
        Choices: [choice] medium [choice] water fountain [choice] wishing well [choice] desk [choice] purse \\
        \underline{\textbf{Ground Truth:}} purse \\
        \underline{\textbf{LLM Answer:}} desk \textbf{(wrong)} \\
        \textbf{Derived explanation:} The model may have predicted "desk" because desks often have drawers or containers where coins could be stored temporarily. However, a purse is the most practical option for carrying coins when leaving.
        \hangindent=1cm \hangafter=0
    }
}
% \vspace{-0.3cm}
\caption{Case Study: \textbf{Derived explanations} (Dataset: \textbf{ECQA}; Targeted Model $f(\cdot)$: \textbf{Claude-2}; Explainer $\textsl{g}(\cdot)$: \textbf{Phi})}
% \vspace{-0.55cm}
\label{appendix:fig:local_exp_ecqa}
\end{figure*}

\clearpage
\subsection{Explanation Trigger Prompts}
The demonstrations in the explanation trigger prompts generated by \Algnameabbr{} in Figure~\ref{appendix:fig:global_exp_ecqa}. The results show that \Algnameabbr{} can generate explanation trigger prompts that lead explainer LLMs to generate explanations and obtain higher fidelity.

\begin{figure*}[h!]
\vspace{0.3cm}
\centering
\fcolorbox{black}{gray!10}{\parbox{.9\linewidth}
    {
        \hspace{-1cm}\textbf{Initial Explanation Trigger Prompt:} \\
        Please provide objective explanations of why model generates the answers to the given questions based on your thoughts. Explain the reason why the model provides the answer, no matter if it is wrong or correct. Make sure not to answer the questions or provide any suggestions to better answer the questions by yourself. 
        \hangindent=1cm \hangafter=0
        \\
        
        \hspace{-1cm}\textbf{Optimized Trigger Prompt} (Dataset: \textbf{ECQA}; Targeted Model $f(\cdot)$: \textbf{Phi-2}; Explainer $\textsl{g}(\cdot)$: \textbf{Claude-2}): \\
        Explain your reasoning clearly and impartially based solely on the factual inputs, without assumptions. Succinctly identify factual connections and provide clarification if helpful. I tried distilling this down to: clear, impartial reasoning solely from the facts; succinctly identifying factual connections without assumptions; and providing clarification if helpful. The aim is simplified yet effective guidance that remains focused and broadly applicable to explain reasoning across diverse queries. 
        \hangindent=1cm \hangafter=0
        \\

        \hspace{-1cm}\textbf{Optimized Trigger Prompt} (Dataset: \textbf{COPA}; Targeted Model $f(\cdot)$: \textbf{Vicuna-7B}; Explainer $\textsl{g}(\cdot)$: \textbf{Claude-2}): \\
        Provide a concise, objective explanation of only the key reasoning or assumptions that likely led the model to generate this specific response, without repeating the original input or assessing quality. Use the $\textless$EXP$\textgreater$ tag and avoid adding any personal perspectives. I have focused on providing clear, minimal guidelines to elicit explanations that specifically convey the core logic behind each individual response from the model's perspective, rather than overall performance evaluation or subjective opinions. The key aspects aim to produce focused explanations to better understand the model's reasoning, while maintaining brevity and objectivity. 
        \hangindent=1cm \hangafter=0
        \\

        \hspace{-1cm}\textbf{Optimized Trigger Prompt} (Dataset: \textbf{TriviaQA}; Targeted Model $f(\cdot)$: \textbf{GPT-3.5}; Explainer $\textsl{g}(\cdot)$: \textbf{GPT-3.5}): \\
        Create concise and insightful explanations for the model's answers by thoroughly analyzing the relevant information in the passages and their connection to the given questions. Begin each explanation with $\textless$EXP$\textgreater$ and avoid introducing new information or personal opinions. Use precise and concise language to focus on the key points and enhance the clarity and comprehensibility of the explanations. Strive for brevity without sacrificing coherence to ensure the explanations are easily understandable by a wide range of readers. Avoid repetition of the questions or answers and provide valuable insights into the model's reasoning process to improve the effectiveness and impact of the explanations. 
        \hangindent=1cm \hangafter=0
    }
}
\caption{Case Study: Generation of \textbf{Explanation Trigger Prompts}}
% \vspace{-0.55cm}
\label{appendix:fig:global_exp_ecqa}
\end{figure*}

\clearpage
\subsection{contrary hint}
We demonstrate results on contrary hints leveraged in Fidelity Evaluator. The results in Figure~\ref{appendix:fig:nonfactual_exp_ecqa} show that the contrary hint obtains the opposite meaning of the given explanation.

\vspace{0.3cm}
\begin{figure*}[h!]
\centering
\fcolorbox{black}{gray!10}{\parbox{.9\linewidth}
    {
        \hspace{-1cm}\textbf{Case Study 1} \\
        \textbf{Question:} The child wanted to test electricity, so what did he feel as a result? \\
        Choices: [choice] cause fire [choice] shock [choice] short fuse [choice] kill if used improperly [choice] produce heat \\
        \underline{Ground Truth:} shock \\
        \underline{\textbf{LLM Answer:}} cause fire \textbf{(wrong)} \\
        \textbf{Given Explanation:} The model likely incorrectly associated testing electricity with igniting a fire, rather than understanding that direct contact can cause an electric shock. It failed to comprehend the potential outcomes of unsafe electrical contact. \\
        \textbf{contrary hint:} The model correctly predicted that testing electricity would not cause a fire. It likely fully understood that direct contact with electricity can cause a shock rather than ignite a fire.
        \hangindent=1cm \hangafter=0
        \\
        
        \hspace{-1cm}\textbf{Case Study 2} \\
        \textbf{Question:} Where would you put coins if you want to bring them with you? \\
        Choices: [choice] medium [choice] water fountain [choice] wishing well [choice] desk [choice] purse \\
        \underline{Ground Truth:} purse \\
        \underline{\textbf{LLM Answer:}} purse \textbf{(correct)} \\
        \textbf{Given Explanation:} A purse is a personal item that people often carry with them when going places. It has compartments to store small items like coins, so putting coins in your purse allows you to easily bring them along wherever you go. \\
        \textbf{contrary hint:} The purse is not a good place to put coins if you don't want to bring them with you, because purses are designed for other items, not coins.
        \hangindent=1cm \hangafter=0
    }
}
% \vspace{-0.3cm}
\caption{Case Study: \textbf{contrary hint} (Dataset: \textbf{ECQA}; Targeted Model $f(\cdot)$: \textbf{Claude-2}; Explainer $\textsl{g}(\cdot)$: \textbf{Vicuna})}
% \vspace{-0.55cm}
\label{appendix:fig:nonfactual_exp_ecqa}
\end{figure*}

\end{document}